\documentclass[twoside,11pt]{article}

\usepackage{jmlr2e}

\usepackage{hyperref}       
\usepackage{url}            
\usepackage{booktabs}       
\usepackage{amsfonts}       
\usepackage{nicefrac}       
\usepackage{microtype}      
\usepackage[matit]{echodefs}
\usepackage{thmtools, thm-restate}
\usepackage{todonotes}
\usepackage{dirtytalk}
\usepackage{graphicx}
\usepackage{cleveref}
\usepackage{tcolorbox}
\usepackage[bottom]{footmisc}
\usepackage{enumitem}
\usepackage{float}
\usepackage{algorithm}
\usepackage{algorithmic}
\usepackage{multirow}
\usepackage{soul}
\definecolor{Gray}{gray}{0.9}
\usepackage{colortbl}
\newcommand{\etal}{\textit{et al. }}

\def\bH{{\mathbb L}}
\def\bS{{\mathbb S}}
\def\bE{{\mathbb E}}
\def\bH{{\mathbb H}}

\def\acos{{\mathrm{acos}}}
\def\acosh{{\mathrm{acosh}}}

\newcommand\blfootnote[1]{%
  \begingroup
  \renewcommand\thefootnote{}\footnote{#1}%
  \addtocounter{footnote}{-1}%
  \endgroup
}

\ShortHeadings{Linear Classifiers in Product Space Forms}{Tabaghi \etal}
\firstpageno{1}

\begin{document}

\title{Linear Classifiers in Product Space Forms}

\author{\name Puoya Tabaghi$^{\dagger,\star}$ \email{tabaghi2@illinois.edu}
  \AND
   \name Chao Pan$^\dagger$ \email{chaopan2@illinois.edu}
  \AND
	\name Eli Chien$^\dagger$ \email{ichien3@illinois.edu}
  \AND
	\name Jianhao Peng$^\dagger$ \email{jianhao2@illinois.edu}
  \AND
	\name Olgica Milenkovi\'c$^{\dagger}$\email{milenkov@illinois.edu}
	\AND
  	\addr$^\dagger$Coordinated Science Lab, ECE Department\\
	University of Illinois at Urbana-Champaign, USA \\
}
\editor{ --- }
\blfootnote{$^{\star}$Corresponding author}
\maketitle

\begin{abstract}
Embedding methods for product spaces are powerful techniques for low-distortion and low-dimensional representation of complex data structures. 
Here, we address the new problem of linear classification in product space forms --- products of Euclidean, spherical, and hyperbolic spaces. First, we describe novel formulations for linear classifiers on a Riemannian manifold using geodesics and Riemannian metrics which generalize straight lines and inner products in vector spaces. Second, we prove that linear classifiers in $d$-dimensional space forms of any curvature have the same expressive power, i.e., they can shatter exactly $d+1$ points. Third, we formalize linear classifiers in product space forms, describe the first known perceptron and support vector machine classifiers for such spaces and establish rigorous convergence results for perceptrons. Moreover, we prove that the Vapnik-Chervonenkis dimension of linear classifiers in a product space form of dimension $d$ is \emph{at least} $d+1$. We support our theoretical findings with simulations on several datasets, including synthetic data, image data, and single-cell RNA sequencing (scRNA-seq) data. The results show that classification in low-dimensional product space forms for scRNA-seq data offers, on average, a performance improvement of $\sim15\%$ when compared to that in Euclidean spaces of the same dimension.
\end{abstract}


\section{Introduction}
Many practical datasets lie in Euclidean spaces and are thus naturally represented and processed using Euclidean geometry. Nevertheless, non-Euclidean spaces have recently been shown to provide significantly improved representations compared to Euclidean spaces for various data structures~\citep{bronstein2017geometric} and measurement modalities --- e.g., metric and non-metric~\citep{tabaghi2020hyperbolic,tabaghi2020geometry}. Examples include \emph{hyperbolic spaces}, suitable for representing hierarchical data associated with trees~\citep{nickel2017poincare,sala2018representation,tifrea2018poincar}, as well as human-interpretable images~\citep{khrulkov2020hyperbolic}; \emph{spherical spaces,} which are well-suited for capturing similarities in text embeddings and cycle-structures in graphs~\citep{meng2019spherical,gu2018learning}. Another important development in non-Euclidean representation learning are methods for finding \say{good} mixed-curvature  representations for various types of complex heterogeneous datasets \citep{gu2018learning}. All three spaces considered --- hyperbolic, Euclidean, and spherical have \emph{constant curvatures} but differ in their curvature sign (negative, zero and positive, respectively). 

Despite these recent advances in nontraditional data spaces, almost all accompanying learning approaches have focused on (heuristic) neural networks in constant curvature spaces~\citep{bachmann2020constant,ganea2018hyperbolicNN,chami2019hyperboli,liu2019hyperbolic,shimizu2020hyperbolic,dai2021hyperbolic}. The fundamental building block of these neural networks, the perceptron, has received little attention outside the domain of learning in Euclidean spaces. Exceptions include two studies of linear classifiers (perceptrons and SVMs) in purely hyperbolic spaces~\citep{cho2019large,weber2020robust}. Although discussed within a limited context in~\citep{Skopek2020Mixed-curvature,bachmann2020constant}, classification in product spaces remains largely unexplored, especially from a theoretical point of view and for specific emerging data formats. 

Our contributions are as follows. We address for the first time the problem of designing linear classifiers for product space forms (and more generally, for geodesically complete Riemannian manifolds) with provable performance guarantees. Product space forms arise in a variety of applications in which graph-structured data captures both cycles and tree-like entities; examples of particular interest include social networks, such as the Facebook network for which product spaces reduce the embedding distortion by more than $30\%$ when compared to Euclidean or hyperbolic spaces alone~\citep{gu2018learning}; multiomics datasets which contain information about both cellular regulatory networks and cycles, as discussed in~\citep{tabaghi2020geometry}. An important property of product spaces is that they are endowed with logarithmic and exponential maps which play a crucial role in combining classifiers for simple space forms and establishing rigorous performance results. The key ideas behind our analysis are to define separation surfaces in constant curvature spaces directly, through the use of geodesics on Riemannian manifolds, to introduce metrics that render distances in different spaces compatible with each other, and to integrate them into one signed distance. We supplement these analyses with proofs that demonstrate that linear classifiers for $d$-dimensional space forms shatter $d+1$ points, regardless of the curvature, linear classifiers for $d$-dimensional product space forms shatter \emph{at least} $d+1$ points. A particular distance-based classifier, the product space form perceptron, is extremely simple to implement, flexible and it relies on a small number of parameters. The new perceptron comes with provable performance guarantees established via the use of indefinite kernels and their Taylor series analysis. The outlined proof technique significantly departs from the ones used in purely hyperbolic spaces~\citep{cho2019large,weber2020robust} and it allows for generalizations to SVMs. From the practical point of view, we demonstrate that the proposed product space perceptron offers excellent performance on both synthetic product-space data and real-world datasets, such as the MNIST \citep{lecun1998gradient} and Omniglot~\citep{lake2015human} datasets, but also more complex structures such as CIFAR-100 \citep{krizhevsky2009learning} and single-cell expression measurements \citep{zheng2017massively,Hodgkin2020Lymphoma,PBMCs2020HealthyDonor}, which are of paramount importance in computational biology.

The paper is organized as follows. In \Cref{sec:terms,sec:linear_classifiers_in_space_forms} we review special representations of linear classifiers in $d$-dimensional constant curvature spaces, i.e., Euclidean, hyperbolic and spherical spaces, and prove that distance-based classifiers have the same expressive power, independent on the curvature of underlying space: Their Vapnik-Chervonenkis (VC) dimension equals $d+1$. \Cref{sec:main} contains our main results, a description of an approach for generalizing linear classifiers in space forms to product spaces, the first example of a product space form perceptron algorithm that performs provably optimal classification in a finite number of steps and the first implementation of a product space form SVM algorithm. \Cref{sec:Numerical_Experiments} and the Appendix contain our simulation results, pertaining to synthetic, MNIST, Omniglot, and CIFAR-100 data sets. Proofs are delegated to the Appendix unless the insight gained from them is useful for understanding the described algorithmic solutions.
\section{Linear Classifiers in Euclidean Space} \label{sec:terms}
Finite-dimensional Euclidean spaces are inner product vector spaces over $\mathbb{R}$, the set of reals. In contrast, hyperbolic and (hyper)spherical spaces do not have the structure of a vector space. Therefore, we first have to clarify what linear classification means in spaces with nonzero curvatures. To introduce our approach, we begin by redefining Euclidean linear classifiers in terms of commonly used concepts in differential geometry such as geodesics and Riemannian metrics~\citep{ratcliffe1994foundations}. This novel formulation allows us to $(1)$ present a unified view of the classification procedure in metric spaces that are not necessarily vector spaces; $(2)$ formalize \emph{distance-based} linear classifiers in space forms, i.e., classifiers that label data points based on their \emph{signed distances} to the separation surface (\Cref{sec:linear_classifiers_in_space_forms}); and $(3)$ use the aforementioned classifiers as canonical building blocks for linear classifiers in product space forms discussed in \Cref{sec:main}.
 
In a linear (more precisely, affine) binary classification problem we are given a set of $N$ points in a Euclidean space and their binary labels, i.e., $(x_n, y_n) \in \R^{d} \times \set{-1, 1}$ for $n \in [N] \bydef \set{1, \ldots N}$. The goal is to learn a linear classifier that produces the most accurate estimate of the labels. We define a linear classifier with weight $w \in \R^{d}$ and bias $b \in \R$ as
\begin{equation}\label{eq:Euclidean_classifier}
	l^{\bE}_{b,w}(x) = \mathrm{sgn} (w^\T x + b), 
\end{equation}
where $\norm{w}_2 = 1$ and $l^{\bE}_{b,w}(x)$ denotes the estimated label of $x \in \R^d$ for the given classifier parameters $b,w$. The expression~\eqref{eq:Euclidean_classifier} may be reformulated in terms of a \say{point-line} pair as follows: Let $p$ be any point on the decision boundary and $w$ a normal vector; see~\Cref{fig:linear_classifier_definition} (left). Then, we have
\begin{equation} \label{eq:euclidean_linear}
	l^{\bE}_{b,w}(x) = \  \mathrm{sgn} ( \langle w,  x-p \rangle),
\end{equation}
where $b = -p^\T w$ and $\langle \cdot,\cdot \rangle$ stands for the standard dot product. To see how this definition may be generalized, note that the linear classifier returns the sign of the inner product of tangent vectors of two straight lines, namely
\begin{equation}\label{eq:euclidean_line}
	\gamma_{p,x}(t) = (1-t)p+t x \  \mbox{ and } \gamma_n(t) = p+t w,
\end{equation}
at their point of intersection $p \in \R^{d}$ (\Cref{fig:linear_classifier_definition}).\footnote{The unique intersection point can be translated without inherently changing the classifier. However, for general Riemannian manifolds, the classifier structure depends on this unique intersection point.} Here, $\gamma_n$ is the normal line and $\gamma_{p,x}$ is the line determined by $p$ and the point $x$ whose label we want to determine. Note that these lines are smooth curves parameterized by $t \in [0,1]$ (or an open interval in $\R$), which we interpret as \emph{time}.
\begin{figure}[t]
  \center
  \includegraphics[width=1 \linewidth]{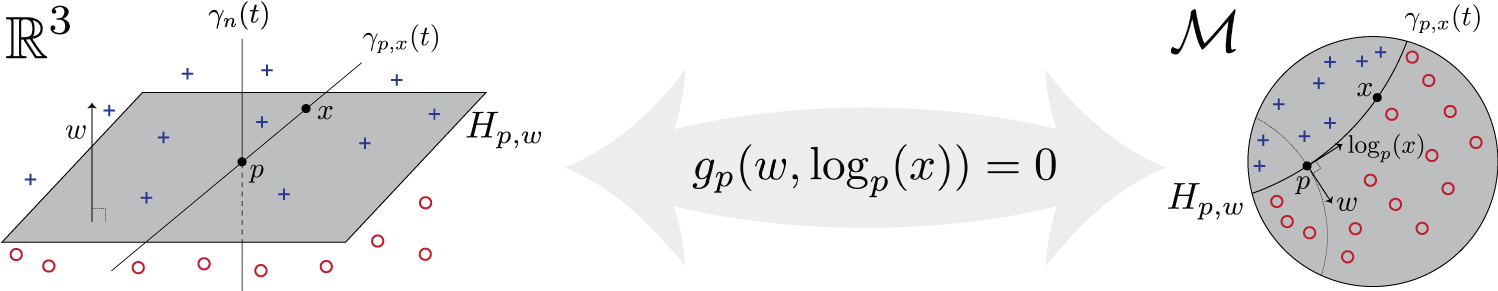}
  \caption{Linear classifiers in a Euclidean space (left) and on a Riemannian manifold (right).} 
  \label{fig:linear_classifier_definition}
\end{figure}

The linear classifier in~\eqref{eq:euclidean_linear} can be reformulated as 
	\[l^{\bE}_{b,w}(x) = \  \mathrm{sgn} ( \langle \frac{d}{dt} \gamma_{p,x}(t)|_{t = 0} , \frac{d}{dt} \gamma_n(t)|_{t = 0} \rangle), \] where the derivative of a line $\gamma(t)$ at time $0$ represents the \emph{tangent vector} (or velocity) at the point $p = \gamma(0)$. This particular formulation leads to the following intuitive definition of linear classifiers in Euclidean spaces, which can be individually generalized for hyperbolic and spherical spaces.
\begin{tcolorbox}[colback=blue!4!white,colframe=blue!4!white]
\begin{definition}\label{def:linear_classifier_euclidean}
A linear classifier in Euclidean space returns the sign of the inner product between tangent vectors of two straight lines described in \eqref{eq:euclidean_line} that meet at a unique point.
\end{definition}
\end{tcolorbox}
Often, we are interested in large-margin Euclidean linear classifiers for which we have $y_n \langle w, x_n-p \rangle \geq \varepsilon$, for all $n \in [N]$, and some margin $\varepsilon>0$. For distance-based classifiers, we want $\varepsilon$ to relate to the distance between the points $x_n$ and the separation surface. For the classifier in~\eqref{eq:euclidean_linear}, the distance between a point $x \in \R^{d}$ and the classification boundary, defined as $H_{p,w} = \set{x \in \R^{d}: \langle w, x-p \rangle = 0}$, can be computed as
\[
	\min_{y \in H_{p,w}}  d(x, y) = \abs{\langle w, x-p \rangle } = \abs{w^\T x + b}.
\]
Note that in the point-line definition~\eqref{eq:euclidean_linear}, the point $p$ can be anywhere on the decision boundary and it has $d$ degrees of freedom whereas $b$ from definition~\eqref{eq:Euclidean_classifier} is a scalar parameter.  Therefore, we prefer definition~\eqref{eq:Euclidean_classifier} as it represents a distance-based Euclidean classifier with only $d+1$ free parameters --- $w$ and $b$ --- and a norm constraint, $\langle w, w\rangle = 1$. In the next section, we show that distance-based classifiers in $d$-dimensional space forms of any curvature can be defined with $d+1$ free parameters and a norm constraint.

\section{Linear Classifiers in Space Forms} \label{sec:linear_classifiers_in_space_forms}
A space form is a complete, simply connected Riemannian manifold of dimension $d \geq 2$ and constant curvature. Space forms are equivalent to spherical, Euclidean, or hyperbolic spaces up to an isomorphism~\citep{lee2006riemannian}. To define linear classifiers in space forms, we first review fundamental concepts from differential geometry such as geodesics, tangent vectors and Riemannian metrics needed to generalize the key terms in~\Cref{def:linear_classifier_euclidean}. For a detailed review; see~\citep{gallier2020differential,ratcliffe1994foundations}. 

Let $\mathcal{M}$ be a Riemannian manifold and let $p \in \mathcal{M}$. The tangent space at the point $p$, denoted by $T_p \mathcal{M}$, is the collection of all tangent vectors at $p$. The Riemannian metric $g_p:T_p \mathcal{M} \times T_p \mathcal{M} \rightarrow \R$ is given by a positive-definite inner product in the tangent space $T_p \mathcal{M}$ which depends smoothly on the base point $p$. A Riemannian metric generalizes the notion of inner products for Riemannian manifolds. The norm of a tangent vector $v \in T_p \mathcal{M}$ is given by $\norm{v} = \sqrt{g_p(v,v)}$. The length of a smooth curve $\gamma: [0,1] \rightarrow \mathcal{M}$ (or path) can be computed as $L[\gamma] = \int_{0}^{1} \norm{\gamma^{\prime}(t)} dt$. A geodesic $\gamma_{p_1,p_2}$ on a manifold is the shortest-length smooth path between the points $p_1, p_2 \in \mathcal{M}$,
\[
\gamma_{p_1,p_2} = \argmin_{\gamma} L[\gamma]: \gamma(0) = p_1, \gamma(1) = p_2;
\]
a geodesic generalizes the notion of a straight line in Euclidean space. Consider a geodesic $\gamma(t)$ starting at $p$ with initial velocity $v \in T_p \mathcal{M}$, e.g., $\gamma(0) = p$ and $\gamma^{\prime}(0) = v$. The exponential map gives the position of this geodesic at $t = 1$, i.e., $\mathrm{exp}_p(v) = \gamma(1)$. Conversely, 
the logarithmic map is its inverse, i.e., $\mathrm{log}_p = \mathrm{exp}_p^{-1} : \mathcal{M} \rightarrow T_p \mathcal{M}$. In other words, for two points $p$ and $x \in \mathcal{M}$, the logarithmic map $\mathrm{log}_p (x)$ gives the initial velocity (tangent vector) with which we can move along the geodesic from $p$ to $x$ in one time step. 

In geodesically complete Riemannian manifolds, the exponential and logarithm maps are well-defined operators. Therefore, analogous to~\Cref{def:linear_classifier_euclidean}, we can define a general notion of linear classifiers as described next.
\begin{tcolorbox}[colback=blue!4!white,colframe=blue!4!white]
\begin{definition} \label{def:linear_classifier_space_form}
Let $(\mathcal{M}, g )$ be a geodesically complete Riemannian manifold, let $p \in \mathcal{M}$ and let $w \in T_p \mathcal{M}$ be a normal vector. A linear classifier $l_{p,w}$ over the manifold $\mathcal{M}$ is defined as
\[
l^{\mathcal{M}}_{p,w}(x) =  \mathrm{sgn} \big(  g_p( w, \mathrm{log}_{p} (x)) \big), \mbox{ where } x \in \mathcal{M}.
\]
\end{definition}
\end{tcolorbox}
\Cref{def:linear_classifier_space_form} is general but it has the following drawbacks: $(1)$ It does not formalize a distance-based classifier since $\abs{ g_p( w, \mathrm{log}_{p} (x))}$ is not necessarily related to the distance of $x$ to the decision boundary; $(2)$ For a fixed $x \in \mathcal{M}$, the decision rule $g_p( w, \mathrm{log}_{p} (x))$ varies with the choice of $p$, which is an arbitrary point on the decision boundary; $(3)$ Often, we can represent the decision boundary with other parameters that have a smaller number of degrees of freedom compared to that of $w$ and $p$ required by \Cref{def:linear_classifier_space_form} (see the Euclidean linear classifiers defined in~\eqref{eq:euclidean_linear} and \eqref{eq:Euclidean_classifier}). We resolve these issues for linear classifiers in space forms as follows.

\subsection{Spherical Spaces}
\begin{table*}[t]\footnotesize
  \caption{Summary of relevant operators in Euclidean, spherical, and hyperbolic ('Loid model) spaces with arbitrary curvatures, $C$.}
  \label{tab:diff_geom_ingredients_with_C}
	\setlength{\tabcolsep}{1.2pt}
  \begin{tabular}{ccccccc}
    \toprule
    $\mathcal{M}$ & $T_p \mathcal{M}$ &$g_p(u,v)$ & $\mathrm{log}_p(x): \theta = \sqrt{|C|} d(x,p)$ & $\mathrm{exp}_p(v)$ & $d(x,p)$ \\
    \midrule
    $\R^d$ & $\R^d$ & $\langle u, v \rangle$ & $x-p$ & $p+v$ & $\norm{x-p}_2$\\
    $\bS^d$ & $p^{\perp}$ & $\langle u, v \rangle$ & $\frac{\theta}{\mathrm{sin}(\theta)}(x -p \cos \theta)$& $\cos (\sqrt{C} \norm{v})p + \sin( \sqrt{C}\norm{v}) \frac{v}{ \sqrt{C}\norm{v}} $  &  $\frac{1}{\sqrt{C}} \mathrm{acos} (C  \langle x,p \rangle)$ \\
    $\bH^d$ & $p^{\perp}$ & $[u,v]$ & $\frac{\theta}{\mathrm{sinh}(\theta)}(x -p \cosh \theta)$& $\cosh ( \sqrt{-C}\norm{v}) p + \sinh( \sqrt{-C}\norm{v}) \frac{v}{ \sqrt{-C}\norm{v}} $ & $\frac{1}{ \sqrt{-C}}\acosh(C [x,p])$\\
  \bottomrule
\end{tabular}
\end{table*}
A $d$-dimensional spherical space with curvature $C_{\bS}>0$ is a collection of points $\mathbb{S}^{d} = \set{x \in \R^{d+1}: \langle x,x \rangle = C_{\bS}^{-1}}$, where $\langle \cdot,\cdot \rangle$ is the standard dot product. Let $p \in \bS^d$ and $w \in T_p \bS^d$ (see~\Cref{tab:diff_geom_ingredients_with_C}). The decision boundary is given by
\begin{align}
H_{p,w} &= \set{ x \in \bS^d: \langle w, \frac{\theta}{\mathrm{sin}(\theta)}  (x -p \cos \theta) \rangle = 0, \mathrm{where } \ \theta = \mathrm{acos}(C_{\bS} \langle x, p \rangle ) } \nonumber \\
&\overset{\mathrm{(a)}}{=}  \set{x \in \bS^d: \langle w, x \rangle =0 } = \bS^d \cap w^{\perp} \label{eq:H_p_w_spherical},
\end{align}
where $\mathrm{(a)}$ is due to the fact that $w \in T_p \bS^d=p^{\perp}$. This formulation uses two parameters $p \in \bS^d$ and $w \in T_p \bS^d$ to define the decision boundary~\eqref{eq:H_p_w_spherical}. We note that one can actually characterize the \emph{same} boundary with fewer parameters. Observe that for any $w \in \R^{d+1}$, we can pick an arbitrary base vector $p \in w^{\perp} \cap \bS^d$ which ensures that $w \in T_p \bS^d$. Therefore, without loss of generality, we can define the decision boundary using only one vector $w \in \R^{d+1}$, which has $d+1$ degrees of freedom. In~\Cref{prop:spherical_classifier}, we identify a specific choice of $p \in w^{\perp} \cap \bS^d$ that allows us to classify each data point based on its signed distance from the classification boundary.
\begin{tcolorbox}[colback=red!4!white,colframe=red!4!white]
\begin{restatable}[]{proposition}{spherical_classifier}  \label{prop:spherical_classifier}
Let $p \in \bS^d$, $w \in T_p \bS^d$, and let $H_{p,w}$ be the decision boundary in~\eqref{eq:H_p_w_spherical}. If $\langle w,w \rangle = C_{\bS}$, then
\[
	\forall x \in \bS^d: \min_{y \in H_{p,w}} d(x,y) = \frac{1}{\sqrt{C_{\bS}}} \mathrm{asin} |\langle w, x\rangle | =  \frac{1}{\sqrt{C_{\bS}}}\abs{g^{\bS}_{p_{\circ}}(w, \mathrm{log}_{p_{\circ}}(x) )},  
\]
where $g^{\bS}$ is the Riemannian metric for a spherical space given in~\Cref{tab:diff_geom_ingredients_with_C}, and $p_\circ =C_{\bS}^{-\frac{1}{2}}\norm{P_w^{\perp}x}^{-1} P_w^{\perp} x  \in H_{p,w}$. Here, the projection operator is defined as $P_{w}^{\perp} x = x - \frac{\langle x, w\rangle }{ \langle w, w \rangle }w$.
\end{restatable}
\end{tcolorbox}
\begin{proof}
Let $p \in \mathbb{S}^d$ and let $w \in T_p \mathbb{S}^d = p^{\perp}$ be such that $\langle w, w\rangle  = C_{\bS}$. The separation surface $H_{p,w}$ is defined as 
\begin{align*}
H_{p,w} = \set{x \in \mathbb{S}^d: g_p( \mathrm{log}_{p}(x),w) = 0} = \set{x \in \mathbb{S}^d: \langle w, x\rangle = 0}.
\end{align*}
The distance between $x \in \bS^d$ and $H_{p,w}$ is given by $d(x, H_{p,w}) = \min_{y \in H_{p,w}} \frac{1}{\sqrt{C_{\bS}}} \mathrm{acos} (C_{\bS} \langle y, x \rangle )$. Hence, the projection of a point $x$ onto $H_{p,w}$ can be computed by solving the following constrained optimization problem:
\[
    \max_{p \in \R^{d+1}} \langle x, p \rangle \ \mbox{ such that } \ \langle w, p \rangle = 0 \ \mbox{,} \ \langle p, p \rangle = C_{\bS}^{-1}, \ \mbox{and} \ \langle w,  w \rangle = C_{\bS}.
\]
From the first-order optimality condition for the Lagrangian, the projected point takes the form $p_{\circ} = \alpha x + \beta w,$ where $\alpha, \beta \in \R$. Now, we impose the following subspace constraint
\begin{align*}
    \langle w,  p_{\circ}\rangle = \alpha \langle w, x \rangle + \beta \langle w, w \rangle = \alpha  \langle w, x \rangle + \beta C_{\bS} = 0,
\end{align*}
which gives $\beta = -\alpha C_{\bS}^{-1} \langle x, w \rangle$. Subsequently, we have $p_{\circ} = \alpha( x - C_{\bS}^{-1} \langle x, w\rangle w)$. On the other hand, from the norm constraint, we have
\begin{align*}
    \langle p_{\circ},p_{\circ} \rangle &= \alpha^2 (C_{\bS}^{-1} + C_{\bS}^{-1} \langle x ,w\rangle^2 - 2 C_{\bS}^{-1} \langle x, w\rangle^2) \\
    &= \alpha^2 C_{\bS}^{-1}(1 - \langle x, w\rangle^2 ) = C_{\bS}^{-1},
\end{align*}
which gives $\alpha = (1-\langle x, w\rangle^2)^{-\frac{1}{2}}$. Then,
\begin{equation} \label{eq:Px_spherical}
	p_{\circ} = \sqrt{\frac{1}{ 1 -\langle x, w\rangle^2 }}( x - \frac{ \langle x, w \rangle }{ \langle w, w \rangle }  w) = \sqrt{\frac{1}{ 1 -\langle x, w \rangle^2 }} P_{w}^{\perp} x =  \frac{C_{\bS}^{-\frac{1}{2}}}{ \norm{P_{w}^{\perp} x }_2}  P_{w}^{\perp} x, 
\end{equation}
where $P_{w}^{\perp}x = x - \frac{\langle x, w \rangle }{ \langle w, w \rangle }w$. Next, let us define $\psi = \acos( \langle x, w \rangle )$, where $\psi \in [0,\pi]$. Then,
\begin{align*}
	d(x,p_{\circ}) &=  \frac{1}{\sqrt{C_{\bS}}}  \acos \big( C_{\bS} \langle x, p_{\circ} \rangle \big) =   \frac{1}{\sqrt{C_{\bS}}}  \acos \big( \sqrt{\frac{1}{ 1 -\cos^2 \psi }}( 1 - \cos^2  \psi ) \big) =  \frac{1}{\sqrt{C_{\bS}}}  \acos \big( \abs{\sin \psi} \big)\\
	&\stackrel{\mathrm{(a)}}{=}  \frac{1}{\sqrt{C_{\bS}}}  \mathrm{asin} \big(\abs{ \cos \psi} \big) =  \frac{1}{\sqrt{C_{\bS}}}  \mathrm{asin} ( |\langle x, w \rangle |),
\end{align*}
where $\mathrm{(a)}$ follows due to the fact that $\acos (\abs{ \mathrm{sin}\psi} ) = \mathrm{asin}( \abs{ \cos \psi})$, which can be seen from
\begin{align*}
	 \mathrm{cos} \big(\mathrm{asin} (\abs{ \cos{\psi}}) \big) &=  \mathrm{cos}\Big( \mathrm{asin}\big( \abs{\sin ({\frac{\pi}{2} - \psi}) } \big) \Big)  = \mathrm{cos} 	\Big(\abs{ \mathrm{asin} (\sin ({\frac{\pi}{2} - \psi}) \big)} \Big)  = \mathrm{cos} \big( \abs{\frac{\pi}{2} - \psi} \big)   \\
	 &=\abs{ \sin \psi},
\end{align*}
for $\psi \in [0, \pi]$. Now, let $x \in \bS^d$ and let $p_{\circ}$ be as given in~\eqref{eq:Px_spherical}. We readily have $p_{\circ}\perp w$. Therefore, 
\begin{align*}
g^{\bS}_{p_{\circ}} \big( w, \mathrm{log}_{p_{\circ}}(x)\big)  &= \frac{\acos ( C_{\bS} \langle p_{\circ}, x \rangle ) }{\sin (\acos(C_{\bS} \langle p_{\circ}, x \rangle )) }  \langle x, w\rangle \stackrel{\mathrm{(a)}}{=}  \frac{\acos (C_{\bS} \langle p_{\circ}, x\rangle ) }{ |\langle x, w \rangle | } \langle x, w \rangle  \\ 
 &=\mathrm{asin} (|\langle x, w\rangle |) \mathrm{sgn}(\langle x, w\rangle ) 
=  \mathrm{asin} ( \langle x, w \rangle ) =  \mathrm{sgn}(\langle x, w\rangle ) \sqrt{C_{\bS}} d(x,p_{\circ}),
\end{align*} 
where $\mathrm{(a)}$ follows from
\begin{align*}
	\sin \big( \acos ( C_{\bS} \langle p_{\circ}, x\rangle ) \big)	&=\sin \big( \acos (  \sqrt{1 - \langle x, w\rangle^2} ) \big) = \sin \big( \mathrm{asin} ( |\langle x, w\rangle | ) \big) = |\langle x, w \rangle |.
\end{align*}
This completes the proof.
\end{proof}
It is important to point out that the classification boundary is invariant with respect to the choice of the base vector, i.e., $H_{p,w} = H_{p_{\circ},w}$. From \Cref{prop:spherical_classifier}, if we have
\[
	\forall n \in [N]: y_n \mathrm{asin}( \langle w, x_n \rangle ) \geq \varepsilon,
\]
then all data points are correctly classified and have a minimum distance of at least $(C_{\bS})^{-\frac{1}{2}}\varepsilon$ to the classification boundary. In summary, we can define distance-based linear classifiers in a spherical space as follows.
\begin{tcolorbox}[colback=blue!4!white,colframe=blue!4!white]
\begin{definition}\label{def:Spherical_classifier}
Let $w \in \R^{d+1}$ with $\langle w, w\rangle = C_{\bS}$. We define a spherical linear classifier as follows
\begin{align*}
	l^{\bS}_{w}(x) = \mathrm{sgn} \big(\mathrm{asin}(\langle w, x\rangle ) \big).
\end{align*}
\end{definition}
\end{tcolorbox}
\subsection{Hyperbolic Spaces}
The 'Loid model of a $d$-dimensional hyperbolic space~\citep{cannon1997hyperbolic}, with curvature $C_{\bH}<0$, is a Riemannian manifold $\mathcal{L}^d= (\bH^{d} , g^{\bH})$ for which 
\[
\bH^{d} = \set{x \in \R^{d+1}: [x,x] = C_{\bH}^{-1} \ \mbox{and} \ x_1 > 0},
\]
and $g_p^{\bH}(u,v)$ corresponds to the Lorentzian inner product of $u$ and $v \in T_p \bH^d$, defined as
\begin{equation} \label{eq:lorentz_inner_product}
	[u, v] = u^\T H v,\ \ H = \begin{pmatrix}
	-1 & 0^\T\\
	0 & I_d
	\end{pmatrix},
\end{equation}
where $I_d$ is the $d \times d$ identity matrix; see~\Cref{tab:diff_geom_ingredients_with_C}. Let $p \in \bH^d$ and $w \in T_p \bH^d$. The classification boundary of interest is given by
\begin{align}
H_{p,w} &= \set{x \in \bH^d: [w,\frac{\theta}{\mathrm{sinh}(\theta)}(x -p \cos \theta)] = 0,  \mbox{where}\ \theta = \mathrm{acosh}(C_{\bH}[x, p])} \nonumber \\ 
&=\set{x \in \bH^d: [w,x] =0 } = \bH^d \cap w^{\perp}. \label{eq:H_p_w_hyperbolic}
\end{align}
Note that equation~\eqref{eq:H_p_w_hyperbolic} follows from the definition of null tangent (sub)spaces below.
\begin{tcolorbox}[colback=blue!4!white,colframe=blue!4!white]
\begin{definition} \label{def:null_tangent_space}
Let $(\mathcal{M},g)$ be a Riemannian manifold. The null tangent space of $V \subseteq T_p \mathcal{M}$ equals $V^{\perp} = \set{u \in T_p \mathcal{M}: g_{p}(u,v) = 0, \ \forall v \in V}$.
\end{definition}
\end{tcolorbox}

Similarly as for the case of spherical spaces, we can simplify` the formulation as follows. If $w$ is a time-like vector --- a vector that satisfies $w \in \set{x: [x,x] > 0}$~\citep{ratcliffe1994foundations} --- and $p \in \bH^d \cap w^{\perp}$, then we have $w \in T_{p}\bH^d$. In \Cref{prop:hyperbolic_classifier}, we describe a special choice for $p \in \bH^d \cap w^{\perp}$ that allows us to formulate a distance-based hyperbolic linear classifier. 
\begin{tcolorbox}[colback=red!4!white,colframe=red!4!white]
\begin{restatable}[]{proposition}{hyperbolic_classifier}  \label{prop:hyperbolic_classifier}
Let $p \in \bH^d$, $w \in T_p \bH^d$, and let $H_{p,w}$ be the decision boundary in \eqref{eq:H_p_w_hyperbolic}. If $[w,w] = -C_{\bH}$, then
\[
	\min_{y \in H_{p,w}} d(x, y ) =  \frac{1}{\sqrt{-C_{\bH}}} \mathrm{asinh} \abs{[w, x]}  = \frac{1}{\sqrt{-C_{\bH}}}\abs{g^{\bH}_{p_\circ}(w, \mathrm{log}_{{p_\circ}}(x) )},
\]
where $g^{\bH}$ is the Riemannian metric for the hyperbolic space given in \Cref{tab:diff_geom_ingredients_with_C} and $p_{\circ} = (-C_{\bH})^{-\frac{1}{2}}\norm{P_w^{\perp} x}^{-1} P_w^{\perp} x  \in H_{p,w}$. Note that $P_{w}^{\perp} x$ is the orthogonal projection of $x$ onto $w^{\perp}$, i.e., $P_{w}^{\perp}x = x - \frac{[x,w]}{[w,w]}w$. 
\end{restatable}
\end{tcolorbox}
\begin{proof}
Let $\bH^d$ be  the 'Loid model with curvature $C_{\bH} < 0$. The projection of $x \in \bH^d$ onto $H_{p,w}$ is a point $p_{\circ} \in H_{p,w}$ that has the smallest distance to $x$. In other words, $p_{\circ}$ is the solution to the following constrained optimization problem
\[
    \max_{p \in \R^{d+1}} \ [y,x] \ \mbox{ such that } [p,p] = C_{\bH}^{-1},\ \ [w,p] = 0 \mbox{, and} \ [w,w] = -C_{\bH}.
\]
The solution to this problem takes the form $p_{\circ} = \alpha x + \beta w$, where $\alpha, \beta \in \R$. We can enforce the subspace condition as follows:
\begin{align*}
    [p_{\circ},w] &= \alpha [x,w] + \beta [w,w] = \alpha [x,w] + \beta(-C_{\bH}) = 0,
\end{align*}
which gives $\beta = \alpha C_{\bH}^{-1} [x,w]$, or $p_{\circ}= \alpha (x +  C_{\bH}^{-1} [x,w] w)$. 
On the other hand, we also have
\begin{align*}
    [p_{\circ}, p_{\circ}] = \alpha^2 (C_{\bH}^{-1} -C_{\bH}^{-1}  [x,w]^2 +2 C_{\bH}^{-1}[x,w]^2) =\alpha^2 C_{\bH}^{-1}(1+ [x,w]^2) = C_{\bH}^{-1}.
\end{align*}
Then,
\begin{equation}\label{eq:Px_hyperbolic}
	p_{\circ} = \sqrt{\frac{1}{1 +[x,w]^2}} (x +C_{\bH}^{-1} [x,w] w) = \frac{(-C_{\bH})^{-\frac{1}{2}}}{\norm{P_{w}^{\perp}x}} P_{w}^{\perp}x,
\end{equation}
where $P_{w}^{\perp}x = x - \frac{[x,w]}{[w,w]}w$ and $\norm{P_w^{\perp}x} = \sqrt{-[P_w^{\perp}x,P_w^{\perp}x]}$. We also have
\begin{align*}
    d(x,p_{\circ}) &= \frac{1}{\sqrt{-C_{\bH}}} \mathrm{acosh} (C_{\bH} [p_{\circ},x]) =  \frac{1}{\sqrt{-C_{\bH}}} \mathrm{acosh} (C_{\bH}\sqrt{\frac{1}{1 + [x,w]^2}} C_{\bH}^{-1}(1 + [x,w]^2) ) \\
    &=  \frac{1}{\sqrt{-C_{\bH}}}\mathrm{acosh} (\sqrt{ 1 + [x,w]^2} ).
\end{align*}
This expression can be further simplified as\footnote{Since $\mathrm{cosh}(x)^2 - \mathrm{sinh}(x)^2 = 1$.} $d(x,p_{\circ})  = \frac{1}{\sqrt{-C_{\bH}}}\, \mathrm{asinh} \abs{[x,w]}$.

Now, let $x \in \bH^d$ and let $p_{\circ}$ be given in \eqref{eq:Px_hyperbolic}. We can easily see that $[p_{\circ},w] = 0$. Therefore, we have
\begin{align*}
	g^{\bH}_{p_{\circ}}(w, \mathrm{log}_{p_{\circ}}(x)) &= \frac{\mathrm{acosh} (C_{\bH}[p_{\circ}, x]) }{\mathrm{sinh} (\mathrm{acosh}(C_{\bH}[p_{\circ},x])) }  [x, w]  = \frac{\mathrm{asinh}( \abs{[x,w]})}{\abs{[x,w]}} [x,w] \\
	&= \mathrm{asinh} ( \abs{[x,w]} ) \mathrm{sgn} ([x,w]) = \mathrm{asinh} ([x,w]) = \mathrm{sgn} ([x,w]) \sqrt{-C_{\bH}}d(x, p_{\circ}).
\end{align*}
This completes the proof.
\end{proof}
As a result, we have the following definition of distance-based linear classifiers in a hyperbolic spaces.
\begin{tcolorbox}[colback=blue!4!white,colframe=blue!4!white]
\begin{definition}\label{def:hyperbolic_classifier}
Let $w \in \R^{d+1}$ with $[w,w] = -C_{\bH}$. We define a hyperbolic linear classifier as follows
\begin{align*}
	l^{\bH}_{w}(x) = \mathrm{sgn} \big( \mathrm{asinh}( [w, x] ) \big).
\end{align*}
\end{definition}
\end{tcolorbox}
\Cref{fig:linear_classifier_space_form} illustrates linear classifiers in two-dimensional hyperbolic, Euclidean, spherical spaces. The classification criteria for linear classifiers in space forms can be compactly written as follows: 
\[
    l(x) = \left\{\begin{array}{lr}
        \mathrm{sgn} \big( \mathrm{asin}_{C} \circ g_{C}(w, x) \big) & \text{for} \ \ C \neq 0\\
        \mathrm{sgn} \big(  w^{\T}x + b\big) & \text{for} \ \ C = 0
        \end{array} 
        \right. 
\]
where $x$ belongs to a $d$-dimensional space form with curvature $C$ and
\[
	\mathrm{asin}_{C}(\cdot)=\begin{cases}
		\mathrm{asin}(\cdot)  &\text{for} \ \ C > 0\\
		\mathrm{asinh}(\cdot) &\text{for} \ \ C > 0
	\end{cases}
\]
$g_C(\cdot, \cdot)$ computes the standard dot product of its inputs if $C>0$ and the Lorentzian product if $C<0$. The vector $w \in \R^{d+1}$ is such that $g_{C}(w,w) = |C|$.
\begin{figure}[t]
  \center
  \includegraphics[width=1 \linewidth]{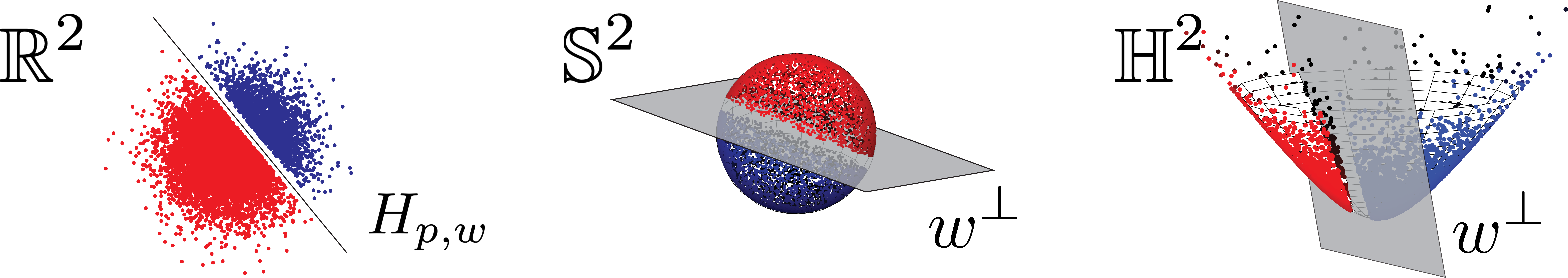}
  \caption{Linear classifiers in Euclidean, spherical, and hyperbolic spaces.}
  \label{fig:linear_classifier_space_form}
\end{figure}
\subsection{VC Dimension of Linear Classifiers in Space Forms}
From the previous discussion, we can deduce that \emph{linear classifiers in $d$-dimensional space forms can be characterized using $d+1$ free parameters and a norm constraint.} This supports the following result pertaining to the VC dimension~\citep{vapnik2013nature} of linear classifiers in space forms.
\begin{tcolorbox}[colback=red!4!white,colframe=red!4!white]
\begin{restatable}[]{theorem}{vp_dimension}  \label{thm:vc_dimension}
The VC dimension of a linear classifier in a space form $S$ is $\mathrm{dim}(S)+1$.
\end{restatable}
\end{tcolorbox}
\begin{proof}
The VC dimension of a linear classifier is equal to the maximum size of a point set that a set of linear classifiers can \emph{shatter}, i.e., completely partition into classes independent on how the points in the set are labeled. The VC dimension of affine classifiers in $\R^d$ is $d+1$; see the treatment of VC dimensions of Dudley classes described in~\citep{dudley1978central}. Note again the distinction between affine and linear classifiers in Euclidean spaces. Now, we establish the VC dimension for spherical and hyperbolic spaces.

Let $x_{1}, \ldots, x_{N}$ be a set of point in spherical space $\bS^d$, which can be shattered by linear classifiers. In other words, we have
\begin{align*}
	\forall n \in [N]: \ y_n =\mathrm{sgn} \big( \mathrm{asin}( \langle w, x_{n} \rangle \big), 
\end{align*}
for any set of binary labels $(y_n)_{n \in [N]}$. The linear classifiers in spherical space are a subset of linear classifiers in a $(d+1)-$dimensional Euclidean space. Hence, their VC dimension must be $\leq d+1$. On the other hand, we can project a set of $d+1$ points that can be shattered by linear classifiers in $\R^{d+1}$ onto $\bS^d$ through a simple normalization. This way, we can find a set of exactly $d+1$ points that can be shattered by linear classifiers in $\S^d$. Hence, the VC dimension of linear classifiers in $\bS^d$ is exactly $d+1$.

Next, we turn our attention to the $d$-dimensional 'Loid model of hyperbolic spaces. Let $x_1, \ldots, x_{d+1}$ be a set of $d+1$ points in this space such that
\[
	x_n = \begin{bmatrix}
	\sqrt{1+\norm{z_n}^2} \\
	z_n
	\end{bmatrix}, \  \mbox{where} \ z_n \in \R^d \ \mbox{for all} \ n \in [d+1].
\]
Furthermore, we assume that $z_1 = 0$ and $z_n = e_{n-1}$ for $n \in \set{2, \ldots, d+1}$, where $e_n$ is the $n$-th standard basis vector of $\R^d$. We claim that this point set can be shattered by the set of linear classifiers in hyperbolic spaces, i.e.,
\begin{equation}\label{eq:hyperbolic_classifiers}
	\forall n \in [d+1]: l^{\bH}_{w}(x_n)= \mathrm{sgn} \big( \mathrm{asinh}([w,x_n]) \big) = y_n,
\end{equation}
where $w \in \set{x \in \R^{d+1}: [x,x] > 0}$ and $(y_1, \ldots, y_{d+1})$ is an arbitrary set of labels in $\set{-1,1}$. We define
\begin{equation} \label{eq:t_and_y}
	\forall n \in \set{2, \ldots, d+1}: t_1 = y_1, \ t_n = k y_n,
\end{equation}
where $k > \sqrt{2}+1$. Therefore, if we can show that there exists a $w \in \set{x \in \R^{d+1}: [x,x] > 0}$ such that $t_n =  [w,x_n]$ for all $n \in [d+1]$, then \eqref{eq:hyperbolic_classifiers} holds true. This is equivalent to showing that the equation $t = X^\T H w$ has a solution $w \in \set{x: [x,x] > 0}$, where $H$ is defined in~\eqref{eq:lorentz_inner_product}, $t = (t_1, \ldots, t_{d+1})$, and 
\[
    X^\T = \begin{bmatrix}
                \sqrt{1+\norm{z_1}^2} & z_1^\T \\
                \sqrt{1+\norm{z_2}^2} & z_2^\T \\
                \vdots& \vdots \\
                \sqrt{1+\norm{z_{d+1}}^2} & z_{d+1}^\T \\
            \end{bmatrix} = \begin{bmatrix}
                                1 & 0^\T \\
                                \sqrt{2} & e_1^\T \\
                                \vdots& \vdots \\
                                \sqrt{2} & e_{d}^\T \\
                            \end{bmatrix}.
\]
The solution is $w = H(X^\T)^{-1}t$, described below,
\begin{align*}
    w = H \begin{bmatrix}
                1 & 0^\T \\
                -\sqrt{2} & e_1^\T \\
                \vdots& \vdots \\
                -\sqrt{2} & e_{d}^\T \\
                \end{bmatrix}t = \begin{bmatrix}
                                    -t_1 \\
                                    -\sqrt{2}t_1 + t_2 \\
                                    \vdots \\
                                    -\sqrt{2}t_1 + t_{d+1} \\
                                  \end{bmatrix}.
\end{align*}
As the final step, we show that $w \in \set{x: [x,x] > 0}$. To this end, we observe that
\begin{align*}
    [w,w] &= -t_1^2 + \sum_{n=2}^{d+1} (-\sqrt{2} t_1 +t_n)^2 \stackrel{\mathrm{(a)}}{=}  y_1^2 \big(-1+ \sum_{n=2}^{d+1} (-\sqrt{2} + k \frac{y_n}{y_1} )^2 \big) \\
    &\stackrel{\mathrm{(b)}}{=}  -1+ \sum_{n=2}^{d+1} (-\sqrt{2} + k \frac{y_n}{y_1} )^2  \stackrel{\mathrm{(c)}}{>} 0, 
\end{align*}
where $\mathrm{(a)}$ is due to~\eqref{eq:t_and_y}, $\mathrm{(b)}$ follows from $y_n \in \set{-1,1}$, and $\mathrm{(c)}$ is a consequence of $k > \sqrt{2}+1$. Therefore, linear hyperbolic classifiers can generate any set of labels for the point set $\set{x_n}_{n \in [d+1]}$. Furthermore, hyperbolic classifiers in \eqref{eq:hyperbolic_classifiers} can be seen as linear classifiers in $(d+1)$-dimensional Euclidean space. Hence, the VC dimension of linear classifiers in hyperbolic space is exactly $d+1$.
\end{proof}
From \Cref{thm:vc_dimension} and the fundamental theorem of concept learning~\citep{shalev2014understanding}, it follows that the set of linear classifiers in a $d$-dimensional space form $S$, denoted by $\mathcal{L}$, is probably accurately correctly learnable. More precisely, let $\Delta$ be a family of probability distributions on $S \times \set{-1,1}$, and let $(x_n,y_n)_{n \in [N]}$ be a set of i.i.d. samples from $P \in \Delta$. Then, we have
\begin{align*}
	\inf_{l \in \mathcal{L}} P( \widehat{l}_N(X) \neq Y) \leq \inf_{l \in \mathcal{L}} P(l(X) \neq Y) + c \sqrt{\frac{d+1}{n}} + \sqrt{\frac{2 \log(\frac{1}{\delta})}{n}},
\end{align*}
where $c$ is a constant and $\widehat{l}_N = \argmin_{l \in \mathcal{L}} \frac{1}{N} \sum_{n \in [N]} 1(l(x_n) \neq y_n)$ is the empirical risk minimizer. Therefore, spherical, hyperbolic, and Euclidean linear classifiers have the same learning complexity. Next, we show how the three classifiers, all of which have the same expressive power, can be \say{combined} to define a linear classifier in product space forms. 
\section{Linear Classifiers in Product Space Forms} \label{sec:main}
\Cref{def:linear_classifier_space_form} of linear classifiers applies to geodesically complete Riemannian manifolds. Our focus is on linear classifiers in product space forms which are built from the results presented in \Cref{sec:linear_classifiers_in_space_forms}. We first describe a perceptron algorithm for such spaces that provably learns an optimal classifier for linearly separable points in a finite number of iterations. Then, we extend this learning scheme to large-margin classifiers in product space forms. Consider a product space of Euclidean, spherical, and hyperbolic manifolds, e.g., $(\bE^{d_{\bE}}, g^{\bE})$, $(\bS^{d_{\bS}}, g^{\bS})$, $(\bH^{d_{\bH}}, g^{\bH})$ with sectional curvatures $0, C_{\bS}, C_{\bH}$, respectively. The Euclidean manifold is simply $\R^{d_{\bE}}$ while the hyperbolic space is the 'Loid model. Two observations are in place. First, we choose to work with the 'Loid model rather than the Pincar\'e disk (or other isometric hyperbolic models) as this model is amendable for integration with other space forms of nonnegative curvature. This is due to the fact that deriving a distance-based linear classifier in the Poincar\'e model requires a complicated analysis to identify an appropriate base point (see to \Cref{prop:hyperbolic_classifier}). Second, unlike Euclidean spaces in which the product of two subspaces is still Euclidean, this is not the case for spherical and hyperbolic spaces. For example, $\mathbb{S}^2 \times \mathbb{S}^2 \neq \mathbb{S}^4$ and $\mathbb{H}^2 \times \mathbb{H}^2 \neq \mathbb{H}^4$.

The product manifold $\mathcal{M} = \bE^{d_{\bE}} \times \bS^{d_{\bS}} \times \bH^{d_{\bH}}$ admits a canonical Riemannian metric $g$, called the \emph{product Riemannian metric}. The tangent space of $\mathcal{M}$ at a point $p = (p_{\bE}, p_{\bS}, p_{\bH})$ can be decomposed as~\citep{tu2011introduction}
\begin{equation} \label{eq:tangent_product_space_form}
	T_p \mathcal{M} = \bigoplus_{S \in \set{ \bE,\bS,\bH}} T_{p_S} S^{d_S},
\end{equation}
where the right-hand side expression is the direct sum $\bigoplus$ of individual tangent spaces $T_{p_{\bE}}\bE^{d_{\bE}}$, $T_{p_{\bS}}\bS^{d_{\bS}}$, and $T_{p_{\bH}}\bH^{d_{\bH}}$. The \emph{scaled} Riemannian metric used on $\mathcal{M}$ is 
\begin{equation} \label{eq:general_metric}
g_p(u, v) = \sum_{S \in \set{\bE,\bS,\bH} }\alpha_S g_{p_S}^{S}(u_{S}, v_{S}), 
\end{equation}
where $u = (u_{\bE},u_{\bS},u_{\bH}), v = (v_{\bE},v_{\bS},v_{\bH}) \in T_p \mathcal{M}$, $p = (p_{\bE}, p_{\bS},p_{\bH})$, and $\alpha_{\bE}$, $\alpha_{\bS}$, $\alpha_{\bH}$ are positive weights. The choice of the scaled Riemannian metric in equation~\eqref{eq:general_metric} allows for scaling the distances between two points while keeping geodesics, tangent spaces, exponential and logarithmic maps unchanged (see~\Cref{sec:linear_classifiers_in_space_forms}). Moreover, it resolves the potential \say{distance compatibility} issues that arise from possibly vastly different ranges and variances of each component (e.g., $x_{\bE}$, $x_{\bS}$, and $x_{\bH}$) which could lead to a classification criterion that is dominated by the component with the largest variance. 


Based on our previous discussions, in order to describe linear classifiers on the above manifold $\mathcal{M}$, we first need to identify the logarithmic map (see~\Cref{def:linear_classifier_space_form}). For this purpose, we invoke the following known result that formalizes geodesics, exponential and logarithmic maps on $\mathcal{M}$.
\begin{fact} \citep{gallier2020differential}
\label{fact:product_space_concepts}
Let $\mathcal{M} = \bE^{d_{\bE}} \times \bS^{d_{\bS}} \times \bH^{d_{\bH}}$, with a Riemannian metric given by ~\eqref{eq:general_metric}. Then, the geodesics, exponential, and logarithmic maps on $\mathcal{M}$ are the concatenation of the corresponding maps of the individual space forms, i.e., $\gamma(t) = \big( \gamma_{\bE}(t) , \gamma_{\bS}(t), \gamma_{\bH}(t) \big)$, $\mathrm{exp}_{p}(v) = \big( \mathrm{exp}_{p_{\bE}}(v_{\bE}), \mathrm{exp}_{p_{\bS}}(v_{\bS}), \mathrm{exp}_{p_{\bH}}(v_{\bH}) \big)$, and $\mathrm{log}_{p}(x) = \big( \mathrm{log}_{p_{\bE}}(x_{\bE}),  \mathrm{log}_{p_{\bS}}(x_{\bS}) ,\mathrm{log}_{p_{\bH}}(x_{\bH}) \big)$, where $p = (p_{\bE},p_{\bS},p_{\bH})$, $x = (x_{\bE},x_{\bS},x_{\bH}) \in \mathcal{M}$, $v = (v_{\bE},v_{\bS},v_{\bH}) \in T_p \mathcal{M}$, and $\gamma_{\bE}, \gamma_{\bS}, \gamma_{\bH}$ are geodesics in their corresponding space form.\footnote{The distance between $x,y \in \mathcal{M}$ is given by $d(x,y) =  \big( \sum_{S \in \set{ \bE,\bS,\bH}} \alpha_S^2 d_{S}(x_{S},y_{S})^2\big)^{\frac{1}{2}}$; see~\Cref{tab:diff_geom_ingredients_with_C}. }
\end{fact}
Combining the results regarding distance-based linear classifiers in space forms (\Cref{sec:linear_classifiers_in_space_forms}), the definition of tangent product spaces in terms of the product of tangent spaces in~\eqref{eq:tangent_product_space_form}, and the choice of the Riemannian metrics given in \Cref{tab:diff_geom_ingredients_with_C}, we arrive at the following formulation for a product space linear classifier. 
\begin{tcolorbox}[colback=red!4!white,colframe=red!4!white]
\begin{proposition} \label{prop:linear_classifier_in_general_product_space_forms}
Let $\bS^{d_{\bS}}$ and $ \bH^{d_{\bH}}$ be space forms with curvatures $C_{\bS} >0$, and $C_{\bH} <0$. Let $\mathcal{M} = \bE^{d_{\bE}} \times \bS^{d_{\bS}} \times \bH^{d_{\bH}}$ with the metric given by~\eqref{eq:general_metric}. We define a linear classifier on $\mathcal{M}$ as
\[
	l^{\mathcal{M}}_w(x) = \mathrm{sgn} \big( \langle w_{\bE}, x_{\bE} \rangle + \alpha_{\bS}\mathrm{asin}( \langle w_{\bS}, x_{\bS}\rangle ) + \alpha_{\bH}\mathrm{asinh}( [w_{\bH},  x_{\bH}]) +b \big),
\]
where $w_{\bE}, w_{\bS}$ and $w_{\bH}$ have norms of $\alpha_{\bE}$, $\sqrt{C_{\bS}}$, and $\sqrt{-C_{\bH}}$, respectively, and $w$ concisely represents all parameters involved.
\end{proposition}
\end{tcolorbox}
\begin{figure}[t]
  \center
  \includegraphics[width=.9 \linewidth]{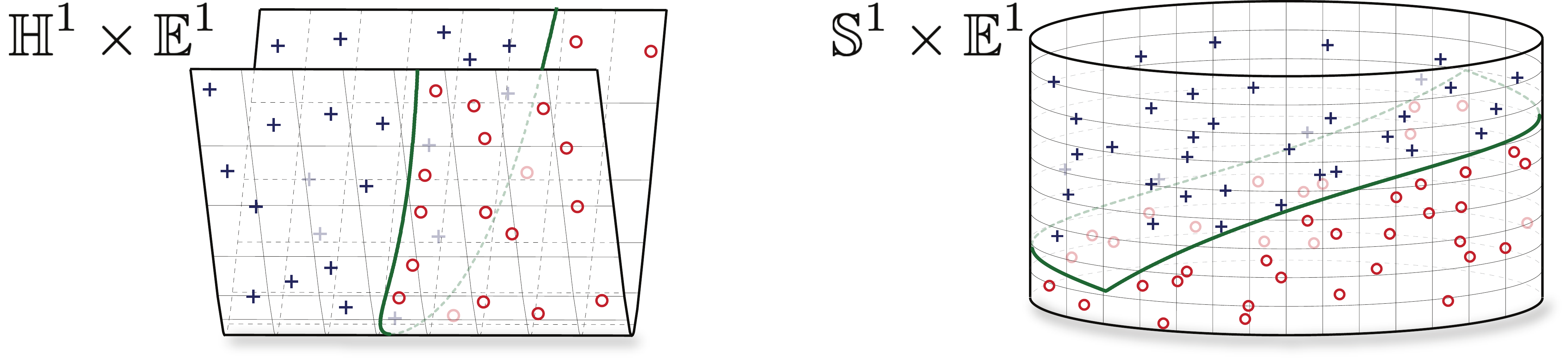}
  \caption{Classification boundaries for linear classifiers in product space forms. }
  \label{fig:product_space_classifiers_margin} 
\end{figure}
\begin{proof}
Let $\mathcal{M} = \mathbb{E}^{d_{\mathbb{E}}} \times \mathbb{S}^{d_{\mathbb{S}}} \times \mathbb{H}^{d_{\mathbb{H}}}$ be a product space with the Riemannian metric $g = \alpha_{\bE} g^{\bE}+ \alpha_{\bS} g^{\bS}+ \alpha_{\bH} g^{\bH}$.
Fact 1 gives us the logarithm map and tangent space at a point $p = (p_{\bE},p_{\bS}, p_{\bH}) \in \mathcal{M}$. A tangent vector $w \in T_{p}\mathcal{M}$ can be expressed as $w = (w_{\bE}, w_{\bS}, w_{\bH})$ where $w_{\bE} \in T_{p_{\bE}} \bE^{d_{\bE}}$, $w_{\bS} \in T_{p_{\bS}} \bS^{d_{\bS}}$, and $w_{\bH} \in T_{p_{\bH}} \bH^{d_{\bH}}$. From the point-line definition of linear classifiers (Definition 2), we have
\begin{align*}
l_{p,w}^{\mathcal{M}} (x) &= \mathrm{sgn} \big( g_{p} (\mathrm{log}_p(x), w ) \big) \\
&= \mathrm{sgn} \big(\alpha_{\bE} g^{\bE}_{p_{\bE}} (\mathrm{log}_{p_{\bE}}(x_{\bE}), w_{\bE} ) + \alpha_{\bS}g^{\bS}_{p_{\bS}} (\mathrm{log}_{p_{\bS}}(x_{\bS}), w_{\bS} )  + \alpha_{\bH} g^{\bH}_{p_{\bH}} (\mathrm{log}_{p_{\bH}}(x_{\bH}), w_{\bH} ) \big).
\end{align*}
In~\Cref{prop:spherical_classifier,prop:hyperbolic_classifier}, we derived specific spherical and hyperbolic base points to formalize distance-based classifiers. From these results, we may define a linear classifier in $\mathcal{M}$ that is parameterized only with a tangent vector $w$, i.e.,
\begin{equation}\label{eq:linear_classifier_in_M}
l_{w}^{\mathcal{M}} (x)= \mathrm{sgn} \big( (\alpha_{\bE} w_{\bE}) ^\T x_{\bE} +b + \alpha_{\bS} \mathrm{asin} ( \langle w_{\bS}, x_{\bS} \rangle ) + \alpha_{\bH} \mathrm{asinh}([w_{\bH},x_{\bH}] ) \big)
\end{equation}
where $\norm{w_{\bE} } = 1$, $\langle w_{\bS}, w_{\bS} \rangle = C_{\bS}$, and $[w_{\bH},w_{\bH}] = -C_{\bH}$. This completes the proof. 
\end{proof}
This classifier can be associated with three linear classifiers, Euclidean, hyperbolic, and spherical space classifiers. For a point $x = (x_{\bE},x_{\bS}, x_{\mathbb{H}}) \in \mathcal{M}$, the product space classifier takes a weighted vote based on the signed distances of each component (e.g, $x_{\bE}$, $x_{\bS}$, and $x_{\mathbb{H}}$) to its corresponding classifier's boundary.\footnote{It is worth mentioning that the Euclidean metric scale $\alpha_{\bE}$ has been absorbed in the norm of $w_{\bE}$.} Two illustrative examples of such classifiers are given in \Cref{fig:product_space_classifiers_margin}.\footnote{Non-Euclidean spaces have dimensions $\geq 2$, but we reduced them to one for visualization purposes only.}

{\bf Remark.} The linear classifier of~\eqref{eq:linear_classifier_in_M} is not a distance-based classifier with respect to our choice of the Riemannian metric $g$. The distance between a point $x$ and the classification boundary $H_{p,w}$ can be computed as
\begin{align*}
\min_{y \in H_{p,x}} d(x, y ) = \Big( \alpha^2_{\bE} \norm{x_{\bE}-y^*_{\bE}}^2+ \frac{ \alpha^2_{\bS}}{C_{\bS}} \mathrm{acos}^2 ( C_{\bS} \langle x_{\bS}, y^*_{\bS} \rangle )+ \frac{\alpha^2_{\bH}}{-C_{\bH}} \mathrm{acosh}^2 ( C_{\bH}[x_{\bH}, y^*_{\bH}]) \Big)^{1/2}
\end{align*}
where $y^*= (y^{*}_{\bE},y^{*}_{\bS}, y^{*}_{\mathbb{H}})$ is the projection of $x$ onto the separation plane $H_{p,w}$. It is easy to verify that this distance is not related to the decision criteria, i.e., $(\alpha_{\bE} w_{\bE}) ^\T x_{\bE} +b + \alpha_{\bS} \mathrm{asin} ( \langle w_{\bS}, x_{\bS} \rangle ) + \alpha_{\bH} \mathrm{asinh}([w_{\bH},x_{\bH}] )$, which only takes the weighted sum of signed distances between $x_{S}$ and $H_{p_{S}, w_{S}}$ for $S \in \set{\bE, \bS, \bH}$. For such a classifier, the classification margin is \emph{the sum ($\ell_1$ norm) of the distances of the individual space components to its classification boundary}, which is related to the weighted vote majority classification approach of~\Cref{sec:main}. The resulting distance is a proper upper bound for the \emph{true} distance of a point to the classification boundary (see Fact~\ref{fact:product_space_concepts}).

\subsection{VC Dimension of Linear Classifiers in Product Space Forms}
In \Cref{thm:mixed_vc} that follows, we provide a lower bound for the VC dimension of linear classifiers in product space forms, introduced in \Cref{prop:linear_classifier_in_general_product_space_forms}, which only depends on the dimension of the ambient manifold. Due to its technical nature, the proof is delegated to \Cref{sec:mixed_vc_proof}.
\begin{tcolorbox}[colback=red!4!white,colframe=red!4!white]
\begin{theorem} \label{thm:mixed_vc}
The VC dimension of linear classifiers in a product space form $\mathcal{M}$ is at least $\mathrm{dim}(\mathcal{M})+1$.
\end{theorem}
\end{tcolorbox}
From \Cref{thm:vc_dimension} and \Cref{thm:mixed_vc}, it is clear that linear product space form classifiers are \emph{at least} as expressive as linear classifiers in simple space forms. To complete our analysis, we compute an upper bound on the VC dimension of product space form classifiers. The key idea behind our approach is to view the classifiers in simple space forms as weak learners which boost the classifier in the product space form.  \Cref{prop:upperbound_vc} shows that, unlike the result of the lower bound which depends only on the dimension of the space, the upper bound depends on the \emph{signature} of the space as well. The signature of the space is the collection of dimensions of each simple space form. As an example, the space $ \mathbb{S}^2 \times \mathbb{E}^{3} \times \mathbb{H}^5$ has signature $(d_{\bS}, d_{\bE}, d_{\bH}) = (2,3,5)$, while the space $\mathbb{S}^4 \times \mathbb{H}^{2} \times \mathbb{H}^3$ has signature $(d_{\bS}, d_{\bE}, d_{\bH}) = (4,0,2 \times 3)$.
\begin{tcolorbox}[colback=red!4!white,colframe=red!4!white]
\begin{proposition}\label{prop:upperbound_vc}
Let $\mathcal{M} $ be a product of space forms with dimensions at least two (allowing for repetitions of the same space form), i.e., 
\[
\mathcal{M} = \Big( \bigtimes_{k=1}^{K-1} S_k^{d_k} \Big) \times \bE^{d_K}  \ \mbox{where} \ S_k \in \set{\bS, \bH} \ \mbox{and} \ d_{k} \geq 2 \ \mbox{for all} \ k \in [K].
\]
The VC dimension of linear classifiers in $\mathcal{M}$ is upper bounded by 
\begin{equation}\label{eq:vc_upper_bound}
\max \set{N \in \mathbb{N}: \frac{N}{\log_2 N} \leq \sum_{k=1}^{K}  ({d_{k}+1})} 
\end{equation}
\end{proposition}
\end{tcolorbox}
\begin{proof}
In \Cref{thm:vc_dimension}, we proved that the VC dimension of linear classifiers in any $d$-dimensional space form equals $d+1$. Now, let $x = (x^1, \ldots, x^K)$ be a point in $\mathcal{M}$, where $x^k \in S_k$ for $k \in [K]$. Suppose that we are given a $K$-dimensional vector, $(h_{1}(x^1), \ldots, h_{K}(x^k))$, where $h_k$ denotes a linear classifier in the space form $S_k$. From the Sauer-Shelah lemma \citep{sauer1972density,shelah1972combinatorial}, given a set of $N$ points $x_1, \ldots, x_N \in \mathcal{M}$, one can generate $\leq  (\frac{eN}{d_{1}+1})^{d_{1}+1}  \ldots (\frac{eN}{d_{K}+1})^{d_{K}+1}$ distinct vectors in $\mathbb{R}^{K}$ using the above set of linear classifiers. For any vector generated in $\mathbb{R}^{K}$, the weighted majority vote classification rule \eqref{eq:linear_classifier_in_M} generates only one label. If the VC dimension of linear classifiers in $\mathcal{M}$ equals $N$, then we must be able to generate $2^N$ possible labels. Therefore,
\begin{align*}
N &\leq \log_2 (\frac{eN}{d_{1}+1})^{d_{1}+1}  \cdots (\frac{eN}{d_{K}+1})^{d_{K}+1} = \sum_{k=1}^{K}  ({d_{k}+1})\log_2 \frac{eN}{d_{k}+1} \\
&=  \sum_{k=1}^{K}  ({d_{k}+1})\log_2 \frac{e}{d_{k}+1} +\sum_{k=1}^{K}  ({d_{k}+1})\log_2 N \stackrel{\mathrm{(a)}}{\leq}  \sum_{k=1}^{K}  ({d_{k}+1})\log_2 N,
\end{align*}
where $\mathrm{(a)}$ is due to the assumptions on the dimension of simple space forms ($d_k \geq 2$). Note that linear product space forms classifiers in $\mathcal{M}$ must shatter at least two points. Hence, we have $\log_2 N \geq 1$. Therefore, the proposed upper bound in \eqref{eq:vc_upper_bound} is $> \mathrm{dim}( \mathcal{M})+1$. This completes the proof.
\end{proof}
\subsection{A Product Space Form Perceptron} 
We now turn our attention to an algorithm for training linear classifiers defined in~\Cref{prop:linear_classifier_in_general_product_space_forms}. To establish provable performance guarantees, we assume that the set of labeled training data points $\mathcal{X}$ satisfies the $\varepsilon>0$ margin property, i.e., 
\begin{equation}\label{eq:mixed_curv_classifier}
\forall (x,y) \in \mathcal{X}: y \big( w_{\bE}^\T x_{\bE}+b + \alpha_{\bS} \mathrm{asin}(w_{\bS}^\T x_{\bS})+  \alpha_{\bH} \mathrm{asinh}( [w_{\bH},  x_{\bH}] \big) \geq \varepsilon, 
\end{equation}
where $\norm{w_{\bE}}_2 = \alpha_{\bE}$, $\norm{w_{\bS}}_2 = \sqrt{C_{\bS}}$, and $\sqrt{[w_{\bH}, w_{\bH}]} =\sqrt{-C_{\bH}}$. 

The classification criterion \eqref{eq:mixed_curv_classifier} is a nonlinear function of the parameters $w_{\bE}$, $w_{\bS}$, $w_{\bH}$, and it requires equality constraints for all the parameters involved. To analyze the classifier and allow for sequential updates of its parameters, we relax the norm constraints and propose perceptron updates in a reproducing kernel Hilbert space (RKHS) which we denote by $\mathcal{H}$.\footnote{The kernel approach is only used to establish convergence results and is not a part of the algorithm.} In the first step of this approach, we represent the classifier in~\eqref{eq:mixed_curv_classifier} as an inner product $\langle \cdot, \cdot \rangle_{\mathcal{H}}$ of two vectors in $\mathcal{H}$, i.e., $l_{w}^{\mathcal{M}}(x)= \mathrm{sgn}\big( \langle \phi(w), M\psi(x) \rangle_{\mathcal{H}} \big),$ where $M$ is a linear operator, $\psi$ and $\phi$ are two maps that are precisely defined in \Cref{sec:perceptron_proof} and discussed in some detail in what follows. 

Note that the kernels $K_{\bE}(w_{\bE},x_{\bE}) = w_{\bE}^\T x_{\bE}$ and $K_{\bS}(w_{\bS},x_{\bS}) = \mathrm{asin}( w_{\bS}^\T x_{\bS} )$ are symmetric and positive definite.\footnote{The Euclidean classifier can be written as $l^{\bE}_{w_{\bE}, b}(x_{\bE}) = \mathrm{sgn} \big( \langle [w_{\bE}, b ], [x_{\bE}, 1] \rangle \big)$.} Hence, they lend themselves to the construction of a valid RKHS. As an example, for spherical spaces, we can write
\[
	K_{\bS}(w_{\bS},x_{\bS}) = \langle \phi_{d_{\bS}}( \frac{1}{\sqrt{C_{\bS}}}w_{\bS}), \phi_{d_{\bS}}(\sqrt{C_{\bS}} x_{\bS}) \rangle_{\mathcal{H}_{d_{\bS}}},
\]
where $\mathcal{H}_{d_{\bS}}$ is a Hilbert space of functions $\bS^{d_{\bS}} \to \R$ equipped with inner product $\langle \cdot , \cdot \rangle_{\mathcal{H}_{d_{\bS}}}$.  Unfortunately, $K_{\bH}(w_{\bH},x_{\bH}) = \mathrm{asinh}( [w_{\bH}, x_{\bH}] )$ is an indefinite kernel. Nevertheless, \Cref{lem:taylor_series_kernel} describes a straightforward approach for finding a similar representation for this kernel. 
\begin{lemma} 
Let $K_{\bS}(x_{\bS}, w_{\bS}) = \mathrm{asin} (x_{\bS}^\T w_{\bS})$, where $x_{\bS}, w_{\bS} \in B_{d_{\bS}} \bydef \set{x \in \R^{d_{\bS}+1}: \norm{x}_2 \leq 1}$. Then, there exists a Hilbert space $\mathcal{H}_{d_{\bS}}$ and a mapping $\phi_{d_{\bS}}:  B_{d_{\bS}} \rightarrow \mathcal{H}_{d_{\bS}}$ such that
\[
	K_{\bS}(x_{\bS}, w_{\bS}) = \langle \phi_{d_{\bS}}(x_{\bS}), \phi_{d_{\bS}}(w_{\bS})\rangle_{\mathcal{H}_{d_{\bS}}},
\]
where $\langle \cdot, \cdot \rangle_{\mathcal{H}_{d_{\bS}}}$ is the inner product on $\mathcal{H}_{d_{\bS}}$.

Moreover, there is exists a Hilbert space $\mathcal{H}_{d_{\bH}}$, a mapping $\phi_{d_{\bH}}:  B_{d_{\bH}} \rightarrow \mathcal{H}_{d_{\bH}}$, and an indefinite operator $M_{d_{\bH}}: \mathcal{H}_{d_{\bH}} \rightarrow \mathcal{H}_{d_{\bH}}$ which admits the following kernel representation
\[
	K_{\bH}(w_{\bH}, x_{\bH}) = \mathrm{asinh}(w_{\bH}^\T H x_{\bH}) = \langle \phi_{d_{\bH}}(w_{\bH}), M_{d_{\bH}} \phi_{d_{\bH}}(H x_{\bH})\rangle_{\mathcal{H}_{d_{\bH}}},
\]
for all $w_{\bH}, x_{\bH}$ in $B_{d_{\bH}}$, and satisfies $M_{d_{\bH}}^\T M_{d_{\bH}} = \mathrm{Id}$, where $\mathrm{Id}$ denotes the identity operator on $\mathcal{H}_{d_{\bH}}$.
\label{lem:taylor_series_kernel}
\end{lemma}
\begin{proof}
The Taylor series expansion of $\mathrm{asin}(\cdot)$ can be used to write
\begin{equation}\label{eq:asin_taylor}
	 \mathrm{asin}(w_{\bS}^\T x_{\bS}) = \sum_{n=0}^{\infty} \frac{(2n)!}{2^{2n}(n!)^2(2n+1)}(w_{\bS}^\T x_{\bS})^{2n+1},
\end{equation}
where $w_{\bS}, x_{\bS} \in B_{d_{\bS}}$. All the coefficients of this Taylor series are nonnegative. Hence, from Theorem 2.1 in~\citep{steinwart2001influence}, this function is a valid positive-definite kernel. Therefore, there is a Hilbert space $\mathcal{H}_{d_{\bS}}$ endowed with an inner product $\langle \cdot, \cdot \rangle_{\mathcal{H}_{d_{\bS}}}$ such that 
\[
	\mathrm{asin}(w_{\bS}^\T x_{\bS})  = \langle \phi_{d_{\bS}}(w_{\bS}), \phi_{d_{\bS}}(x_{\bS})\rangle_{\mathcal{H}_{d_{\bS}}},
\]
for $w_{\bS}, x_{\bS} \in B_{d_{\bS}}$ and vectors $\phi_{d_{\bS}}(w_{\bS}), \phi_{d_{\bS}}(x_{\bS}) \in \mathcal{H}_{d_{\bS}}$.

On the other hand, we have
\[
	\mathrm{asinh}(w_{\bH}^\T H x_{\bH}) =\sum_{n=0}^{\infty} (-1)^n \frac{(2n)!}{2^{2n}(n!)^2(2n+1)}(w_{\bH}^\T H x_{\bH})^{2n+1}, 
\]
where $w_{\bH}, x_{\bH} \in B_{d_{\bH}}$. This Taylor series is nearly identical to the one given in~\eqref{eq:asin_taylor}, except for the alternating signs of the coefficients. The analytical construction of the vector $\phi_{d_{\bH}}(x)$ in  \citep{steinwart2001influence} gives a straightforward way to define an indefinite operator $M_{d_{\bH}}: \mathcal{H}_{d_{\bH}} \rightarrow \mathcal{H}_{d_{\bH}}$ such that $M_{d_{\bH}}^\T M_{d_{\bH}}= \mathrm{Id}$, and
\begin{align*}
\mathrm{asinh}(w_{\bH}^\T H x_{\bH})  &= \langle \phi_{d_{\bH}}(w_{\bH}), M_{d_{\bH}} \phi_{d_{\bH}}(H x_{\bH}) \rangle_{\mathcal{H}_{d_{\bH}}} \\
&= \langle \phi_{d_{\bH}}(w_{\bH}),  M_{d_{\bH}}\psi_{d_{\bH}}(x_{\bH}) \rangle_{\mathcal{H}_{d_{\bH}}},
\end{align*}
where $\psi_{d_{\bH}}(x_{\bH})  = \phi_{d_{\bH}}(H x_{\bH})$. Note that $M_{d_{\bH}}$ can be represented as an infinite-dimensional diagonal matrix with diagonal elements $\pm 1$ that capture the signs of the Taylor series coefficients. This completes the proof.
\end{proof}

\Cref{lem:taylor_series_kernel} shows that spherical and hyperbolic kernels can be represented as inner products of \emph{feature points} and \emph{feature parameters} in a Hilbert space. Further analysis, detailed in \Cref{sec:perceptron_proof}, allows us to write the classifier \eqref{eq:mixed_curv_classifier} as 
\begin{equation} \label{eq:decision_rule_RKHS}
	l_{w}^{\mathcal{M}}(x) = \mathrm{sgn} \big( \langle \phi(w), M \psi(x) \rangle_{\mathcal{H}} \big),
\end{equation}
where $\phi(w), \psi(x) \in \mathcal{H}$ are two vectors in the Hilbert space $\mathcal{H}$, $\langle \cdot, \cdot \rangle_{\mathcal{H}}$ is the inner product defined on $\mathcal{H}$, and $M$ is an indefinite linear operator such that $M^{\T} M = \mathrm{Id}$. The separable form \eqref{eq:decision_rule_RKHS} allows us to formulate the update rule for the perceptron in $\mathcal{H}$ as
\begin{equation} \label{eq:update_product_space}
	\phi_{\mathcal{H}}^{k+1} = \phi_{\mathcal{H}}^{k} + y_n M \psi(x_n)
\end{equation}
for any misclassified point $x_n$, i.e., any point that satisfies $y_n \langle \phi_{\mathcal{H}}^{k}, M \psi(x_n)\rangle_{\mathcal{H}} < 0$. Since the decision rule~\eqref{eq:decision_rule_RKHS} only depends on the inner products of vectors in $\mathcal{H}$, the classifier trained by the proposed RKHS perceptron computes the label of a given point in $\mathcal{M}$ as described in~\Cref{alg:mixed_curvature_perceptron}. The interested reader is referred to \Cref{sec:perceptron_proof} for the detailed derivation and analysis of \Cref{alg:mixed_curvature_perceptron}.
\begin{algorithm}[!t]
	\caption{\small A Product Space Form Perceptron}
   	\label{alg:mixed_curvature_perceptron}
	\begin{algorithmic}
	\STATE {\bfseries Input:} $(x_n, y_n)_{n=1}^{N}$:  point-label pairs in $\mathcal{M} \times \set{-1, 1}$;
	\STATE {\bfseries Initialization:} $k = 0$, $n = 1$, $x_n \bydef (x_{\bE,n},x_{\bS,n},x_{\bH,n})$, $R = \max_{n \in [N]} \norm{x_{\bH,n}}_{2}$;
	\REPEAT 
	\IF{$ \mathrm{sgn} \big( g_{k}(x_n) \big) \neq y_n $}
	\STATE $K(x,x_n) \bydef 1+x_{\bE}^\T x_{\bE,n}+ \alpha_{\bS}\mathrm{asin}(C_{\bS} x_{\bS}^\T x_{\bS,n}) +  \alpha_{\bH}\mathrm{asin}( R^{-2} x_{\bH}^\T x_{\bH,n})$
	\STATE $g_{k+1}( x ) \leftarrow g_{k}(x) +y_n K(x,x_n)$
	\STATE $k \leftarrow k+1$
	\ENDIF
	\STATE $n \leftarrow \mathrm{mod}(n, N)+1$
   	\UNTIL{convergence criterion is met}
	\end{algorithmic}
\end{algorithm}
In~\Cref{thm:mixed_curvature_perceptron}, we prove that the product space perceptron in~\Cref{alg:mixed_curvature_perceptron} converges in a finite number of steps.
\begin{tcolorbox}[colback=red!4!white,colframe=red!4!white]
\begin{restatable}[]{theorem}{mixed_curvature_perceptron}\label{thm:mixed_curvature_perceptron}
Let $(x_n, y_n)_{n=1}^{N}$ be points in a compact subset of $\mathcal{M}$ with labels in $\set{-1, 1}$, and $\norm{x_{\bH,n}}_2 \leq R$ for all $n \in [N]$. If the point set is $\varepsilon$-margin linearly separable and $\norm{w_{\bH}}_2 \leq 1/R$, then~\Cref{alg:mixed_curvature_perceptron} converges in $O(\frac{1}{\varepsilon^2}) $ steps.
\end{restatable}
\end{tcolorbox}
The constraint $\norm{w_{\bH}}_2 \leq 1/R$ is necessary to establish the convergence result for the proposed update rule. From a theoretical point of view, the bound ensures that parameter features in $\mathcal{H}$ have finite norms, i.e., $\langle \phi(w) , \phi(w) \rangle_{\mathcal{H}} < \infty$ for $\phi(w) \in \mathcal{H}$. However, this norm constraint is not compatible with the definition of product space form classifiers detailed in \Cref{prop:linear_classifier_in_general_product_space_forms}. In practice, we can normalize the hyperbolic component of data points by first modifying $R$ in \Cref{thm:mixed_curvature_perceptron}, then scaling the hyperbolic metric $g^{\bH}$, and subsequently adjusting the weight vector $\alpha_{\bH}$. This process lets us tweak the norm constraint $\norm{w_{\bH}}_2 \leq 1/R$ with a great degree of flexibility to construct classifiers that perform reliably on real-world datasets.

{\bf Related work.} Linear classifiers in spherical spaces have been studied in a number of papers~\citep{novikoff1963convergence,dasgupta2009analysis}. More recent work has focused on linear classifiers in the Poincar\'e model of hyperbolic spaces, and notably, in the context of hyperbolic neural networks~\citep{ganea2018hyperbolicNN}. A purely hyperbolic perceptron was described in~\citep{weber2020robust} but exhibits  converge issues (see \Cref{sec:further_discussions_hyperbolic,sec:Convergence_Analysis_Hyperbolic_Perceptron} for detailed explanations regarding the problems associated with the approach proposed therein). We therefore describe next a modified update rule for a purely hyperbolic perceptron which is of independent interest given many emerging learning paradigms in hyperbolic spaces. Our hyperbolic perceptron uses an appropriate update direction and provably converges, as described below and proved in \Cref{sec:hyperbolic_perceptron_proof}. 
\begin{tcolorbox}[colback=red!4!white,colframe=red!4!white]
\begin{restatable}[]{theorem}{hyperbolic_perceptron}  \label{thm:hyperbolic_perceptron}
Let $(x_n, y_n)_{n=1}^{N}$ be a labeled point set from a bounded subset of $\bH^{d_{\bH}}$. Assume the point set is linearly separable by a margin $\varepsilon$. Then, the hyperbolic perceptron with the update rule $\mathrm{sgn} ( [w^k,x_n] ) \neq y_n: w^{k+1} = w^k + y_n H x_n$, where $H$ is given in \eqref{eq:lorentz_inner_product}, converges in $O \Big( \frac{1}{\mathrm{sinh}^2( \varepsilon)} \Big)$ steps.
\end{restatable}
\end{tcolorbox}
For small classification margin $\varepsilon$, we have $\mathrm{sinh}^2( \varepsilon) \approx \varepsilon^2$. Hence, for borderline linearly separable data points, \Cref{thm:hyperbolic_perceptron} proves that the hyperbolic and Euclidean perceptron have the same convergence rate.

\subsection{A Product Space Form SVM}\label{sec:mixed_curvature_svm}
In the previous section, we showed that the classification criterion for linear classifiers defined in~\Cref{prop:linear_classifier_in_general_product_space_forms} is a linear function of the \emph{feature} vectors, or, more precisely, of $\set{M \psi(x_n)}_{n \in [N]}$. This fact and the subsequent performance guarantees are due to the update rule operating in the RKHS which, in effect, lifts a finite dimensional point to a feature vector. Here, we use the kernel space formalism to formulate large-margin classifiers in product space forms. The idea behind our approach is to use the feature vector representation of linear classifiers. A closed-form expression for the distance between the points and the classification boundary is not available, although it can still be upper bounded as explained the remark following~\Cref{prop:linear_classifier_in_general_product_space_forms}. The described solution complements the prior work on hyperbolic SVMs~\citep{cho2019large,chien2021highly}, and as will be seen from the simulation results, improves upon the first line of work.

Let $x_1, \ldots, x_N \in \mathcal{M}$ be a collection of points. We showed that the decision rule \eqref{eq:decision_rule_RKHS} is a linear function of the feature vectors, i.e., $M\psi(x_n)$. Hence, the \emph{representer theorem}~\citep{scholkopf2001generalized} allows one to express the set of feasible parameters in the space $\mathcal{H}$ as linear combinations of measured feature vectors. More precisely,
\[
\phi( w ) = \sum_{n \in [N]} \beta_n M \psi(x_n), \mbox{ where } \ \sum_{n \in [N]} \beta_n^2 < \infty,
\]
and the parameter vector $w$ is implicit in the expressions for $\phi$ and $\beta$. The classification criterion is a linear function in $\beta = (\beta_1, \ldots, \beta_N)$, i.e.,
\begin{equation}\label{eq:product_space_classifier_kernel}
l_{w}^{\mathcal{M}}(x) = \mathrm{sgn} \big( \sum_{n \in [N]} \beta_n \langle M\psi(x_n), M\psi(x) \rangle_{\mathcal{H}} \big) =  \mathrm{sgn} \big( \sum_{n \in [N]} \beta_n K( x_n, x) \big),
\end{equation}
where $K(x,x_n) = 1+x_{\bE}^\T x_{\bE,n}+ \alpha_{\bS}\mathrm{asin}(C_{\bS} x_{\bS}^\T x_{\bS,n}) +  \alpha_{\bH}\mathrm{asin}( R^{-2} x_{\bH}^\T x_{\bH,n})$. In  \Cref{alg:mixed_curvature_perceptron}, the parameter features (in RKHS) are sequentially updated after each missclassification. Instead, here we directly optimize the weight vector $\beta = (\beta, \ldots, \beta_N)$ to ensure the maximum separability condition. In the following proposition, we derive necessary conditions for the vector $\beta$ that enable distance-based formulations of classifiers in RKHS.
\begin{tcolorbox}[colback=red!4!white,colframe=red!4!white]
\begin{restatable}[]{proposition}{three_conditions}  \label{prop:three_conditions}
For the classifier in~\eqref{eq:product_space_classifier_kernel}, the following claims hold:
\begin{itemize}[leftmargin=*] \small
\item $\langle w_{\bE}, w_{\bE} \rangle = \alpha_{\bE}^2 \Rightarrow \beta^\T K_{\bE} \beta = \alpha_{\bE}^2$, where $K_{\bE} = \big( \langle x_{\bE,i}, x_{\bE,j} \rangle \big)_{i,j \in [N]}$
\item $\langle w_{\bS}, w_{\bS} \rangle = C_{\bS} \Rightarrow \beta^\T K_{\bS} \beta = \frac{\pi}{2}$, where $K_{\bS} = \big( \mathrm{asin}( C_{\bS} \langle x_{\bS,i}, x_{\bS,j}\rangle)  \big)_{i,j \in [N]}$
\item $[ w_{\bH} , w_{\bH}] = -C_{\bH} \Rightarrow \beta^{\T} K_{\bH} \beta = \mathrm{asinh}(-R^2C_{\bH})$,  where $K_{\bH} = \big( \mathrm{asinh}(R^{-2}[ x_{\bH,i}, x_{\bH, j}]  \big )_{i,j \in [N]}$.
\end{itemize}
\end{restatable}
\end{tcolorbox}
The result of~\Cref{prop:three_conditions} lets us define a product space form classifier that satisfies the equality constraints in~\Cref{prop:linear_classifier_in_general_product_space_forms}. To do so, we define the following constraint sets:
\begin{align}
\mathcal{A}_{\bE} &= \set{ x \in \R^{N} : x^\T K_{\bE} x = \alpha_{\bE}^2} \label{eq:beta_eulicean} \\
\mathcal{A}_{\bS} &= \set{x \in \R^{N}: x^\T K_{\bS} x = \frac{\pi}{2} } \label{eq:beta_spherical} \\
\mathcal{A}_{\bH} &= \set{x \in \R^{N}: x^\T K_{\bH} x = \mathrm{asinh}(-R^2C_{\bH}) } \label{eq:beta_hyperbolic}.
\end{align}
In the product space form SVM, we ask for a weight vector $\beta \in \mathcal{A}_{\bE} \cap \mathcal{A}_{\bS} \cap \mathcal{A}_{\bH}$ such that the classification margin is maximized, i.e.,
\[
\max_{\beta  \in \mathcal{A}_{\bE} \cap \mathcal{A}_{\bS} \cap \mathcal{A}_{\bH} } \varepsilon \ \ \mbox{such that} \ \  y\sum_{n \in [N]} \beta_n K( x_n, x) \geq \varepsilon,
\]
for all $(x,y) \in \mathcal{X}$. To convexify the Euclidean and spherical constraint sets \eqref{eq:beta_eulicean} and \eqref{eq:beta_spherical}, we replace them with their convex hulls, i.e.,
\[
\mathrm{convhull}(\mathcal{A}_{\bE}) = \set{ x \in \R^{N} : x^\T K_{\bE} x \leq \alpha_{\bE}^2}, \ \ \mathrm{convhull}(\mathcal{A}_{\bS}) = \set{x \in \R^{N}: x^\T K_{\bS} x \leq \frac{\pi}{2} }.
\]
For the (nonconvex) hyperbolic constraint set \eqref{eq:beta_hyperbolic}, we let $K_{\bH} = K^{+}_{\bH}-K^{-}_{\bH}$ for two positive semidefinite matrices $K^{+}_{\bH}$ and $K^{-}_{\bH}$. Then, we relax the aforementioned set as follows
\begin{equation*}
\widetilde{\mathcal{A}_{\bH}} = \set{x \in \R^{N}: x^{\T} K^{-}_{\bH} x \leq r \ \mbox{and} \  x^{\T} K^{+}_{\bH} x \leq r+\mathrm{asinh}(-R^2C_{\bH}) },
\end{equation*}
where $r$ is a small positive scalar. \Cref{alg:mixed_curvature_svm} summarizes our proposed soft-margin SVM classifier, for points with noisy labels, in product space forms.
\begin{algorithm}[t] 
   \caption{A Product Space Form SVM} \label{alg:mixed_curvature_svm}
\begin{algorithmic}
	\STATE {\bfseries Input:} $(x_n, y_n)_{n=1}^{N}$: a set of point-labels in $\mathcal{M} \times \set{-1, 1}$. Let $r >0$. Then, solve for $\beta$ according to:
	\begin{align*}
	& \text{maximize}       & &  \varepsilon - \sum_{n \in [N]} \zeta_n  &  \\
	& \text{w.r.t}       & &  \varepsilon >0, \set{\zeta_n  \geq 0}&  \\
	& \text{subject to}   &&  \forall n \in [N]: ~~  y_n \sum_{m \in [N]} \beta_{m} K(x_n,x_m) \geq \varepsilon -\zeta_n   \\
	&                   	& & \beta \in \mathrm{convhull} (\mathcal{A}_{\bE}) \cap \mathrm{convhull} (\mathcal{A}_{\bS}) \cap  \widetilde{\mathcal{A}_{\bH}} & 
	\end{align*}
\end{algorithmic}
\end{algorithm} 
\subsection{Signature Estimation}\label{sec:signature_estimation}
One important question that arises in the context of learning in product space forms is \emph{how does one identify the best signature of the embedding space for the task at hand?} In this context, the \emph{optimal} geometry depends on the task-specific performance measure, e.g., classification accuracy, regression error, etc. This type of question has been addressed with limited success in the representation learning literature but only involve simple space forms, e.g., spectral method to estimate the metric signature of graphs~\citep{wilson2014spherical}, discrete version of the triangle comparison theorem for sectional curvature estimation~\citep{gu2018learning}; very little is known about how to find appropriate signatures in product space forms. A notable work, related to node classification tasks, is the \emph{constant curvature graph convolutional network} which allows for a differentiable interpolation between different space forms~\citep{bachmann2020constant}. In most cases, pertaining to simple space forms and arbitrary learning tasks, practitioners in the field heuristically examine a number of signatures to identify one that offers quality performance. Unfortunately, this can not be extended to product space forms due to the combinatorial complexity of possible signatures. A $d$-dimensional product space form can have up to $\sum_{K= 1}^{\frac{d}{2}} 3^K K^{d}$ different signatures.\footnote{Let us assume a $d$-dimensional product space has $K$ simple space forms. There are three choices for each space form and there are at most $\frac{d}{2}$ of these (at least) two-dimensional spaces.}

Nevertheless, a small number of recent works has partially addressed the signature identification problem but in a general context that does not cater to the specific need of a learning task. The authors of~\citep{gu2018learning} showed that one can combine space forms to learn low-dimensional representations for complex graph data with low distortions. In addition, mixed-curvature variational autoencoders (VAEs) have been introduced to streamline non-Euclidean feature extraction~\citep{Skopek2020Mixed-curvature}. These methods can be used to determine the optimal underlying geometry for manifold approximation and unsupervised tasks. Finally, Switch Spaces \citep{zhang2021switch} were proposed to select a mix of $K$ space forms from a given set of $N$ candidate space forms for each data point to be processed. In their proposed approach, space form selections depend on the data points. This makes the signature of the learned product space form switchable depending on the input data and the task-at-hand. 

In our numerical experiments, we use a heuristic bottom-up approach for signature selection. Our greedy algorithm aims at reducing the search space for signatures and relies on two assumptions: (1) We can combine small-dimensional space forms to form large-dimensional (product) space forms, e.g., $\bE^{d_1} \times \bE^{d_2} = \bE^{d_1+d_2}$, $\bS^{d_1} \times \bS^{d_2}$, etc; (2) Linear classifiers in a product space form of dimension $d$ are at least as expressive as linear classifiers in any simple space form of dimension $d$; see ~\Cref{thm:mixed_vc}. The proposed algorithm is illustrated in~\Cref{fig:signature_estimation} and can be summarized as follows: We start with a two-dimensional space form that yields the best classification accuracy compared to other space forms of the same dimension. Then, we consider all possible products of this space form with other two-dimensional space forms. We pick the space which improves upon the classification accuracy computed in the previous step. If the classification accuracy improvement is below a preset threshold, the process terminates. Otherwise, the current product space form is updated by including an additional two-dimensional space form. Note that this greedy approach eliminates space forms $S^{d}$, where $S \in \set{\bS, \bH}$, from the search space. Instead, it may pick product space forms of the type $S^2 \times S^2 \times \ldots$, where $S \in \set{\bS, \bH}$. \Cref{thm:mixed_vc} suggests that the latter types of spaces may provide improved classification performance over simple space forms. The proposed approach reduces the size of the search space to $3^{\frac{d}{2}}$. 

\begin{figure}[t]
  \center
  \includegraphics[width=1 \linewidth]{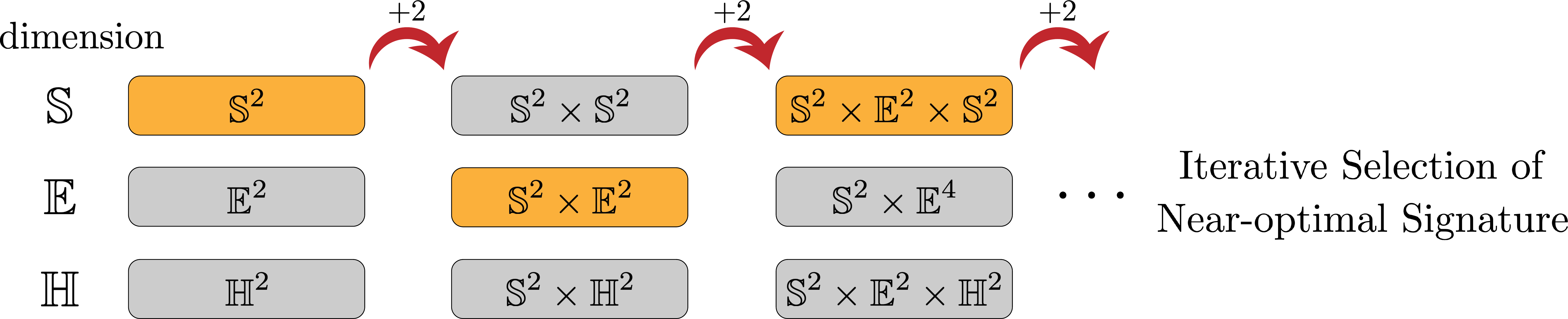}
  \caption{The greedy process for identifying the optimal signature. We first identify the best two-dimensional space form. Then, we compute its product with the best two-dimensional space form that improves upon the representation accuracy. We repeat the process until the increase in accuracy is lower than a preset threshold.}
  \label{fig:signature_estimation}
\end{figure}
\section{Numerical Experiments} \label{sec:Numerical_Experiments}
All  experiments were conducted on a Linux machine with 48 cores, 376GB of system memory. incomplete. Code, datasets, and
documentation needed to reproduce the experimental results are available
at \href{https://github.com/thupchnsky/product-space-linear-classifiers}{https://github.com/thupchnsky/product-space-linear-classifiers}.
\subsection{Synthetic Datasets} \label{sec:Synthetic_Datasets}
We first illustrate the performance of our product space form perceptron (\Cref{alg:mixed_curvature_perceptron}) on synthetic datasets. In order to establish the benefits of product space form embeddings and learning, we compare the performance of the proposed product space form perceptron with the results obtained from a purely Euclidean perceptron. Classification accuracies (macro F1 scores) are reported when training is performed on the entire dataset. In the experiments, data points are sampled from the product space $\bE^{2}\times \bS^{2}\times \bH^{2}$, and viewed as points in $\bE^{8}$ when simulating the Euclidean perceptron.

We generate binary-labeled synthetic data satisfying the $\varepsilon$-margin assumption as follows. First, we randomly generate the optimal decision hyperplane --- with parameter vector $w^*=(w^*_{\bE},0, w^*_{\bS},w^*_{\bH})$  --- under the constraints stated in~\Cref{thm:mixed_curvature_perceptron}. Then, we sample three points $x_{\bE},x_{\bS},x_{\bH}$ from a Gaussian distribution in each of the three space forms, $\bE^{2},\bE^{3},\bE^{3}$; subsequently, we project the latter two points onto $\bS^{2}$ and $\bH^{2}$, with curvatures $+1$ and $-1$, respectively. The points are concatenated to obtain the product space form embedding, i.e., $x = (x_{\bE}, x_{\bS}, x_{\bH})$. Next, we compute the inner product according to~\eqref{eq:mixed_curv_classifier} which ensures that the $\varepsilon$-margin assumption is satisfied and assign the labels accordingly. If the $\varepsilon$-margin assumption is not satisfied, we simply discard the generated point. We repeat this process until $N$ points are generated.

For a fair comparison, with each combination of $(N, \varepsilon)$ we use the same set of points for both Euclidean and product space form perceptrons. This allows us to demonstrate the efficiency and performance gain of our proposed method, which is informed by the geometry of data. In comparison, the Euclidean setting ignores the geometry of data points, i.e., its metric, domain, and other properties, and simply assumes that they lie in $\bE^{8}$. In~\Cref{fig:synthetic_perceptron}, we plot ten experimental convergence curves for different parameter settings, i.e., for $N=100, 200, 300$ points with the fixed optimal decision hyperplane parameterized by vector $w^*$, and with different separation margins, $\varepsilon = 0.01, 0.05, 0.1, 0.2$. From the results, we first observe that the number of updates made by the product space form perceptron (red line) is always smaller than the theoretical upper bound (green dotted line) described in~\Cref{thm:mixed_curvature_perceptron}, \emph{independent} of the size of the datasets. Second, when the margin $\varepsilon$ is small, the dataset generated from the described product space form may not be linearly separable in $\bE^{8}$ and thus the Euclidean perceptron cannot converge to achieve a 100\% accurate solution; this point is further reasserted by the examples of nonlinear classification boundaries in~\Cref{fig:product_space_classifiers_margin}. Note that as the separation margin increases, data points become linearly separable in $\bE^{8}$ as well, in which case the Euclidean perceptron may in some cases converge faster than its product space form counterpart.
\begin{figure}[t!]  
  \center
  \includegraphics[width=1 \linewidth]{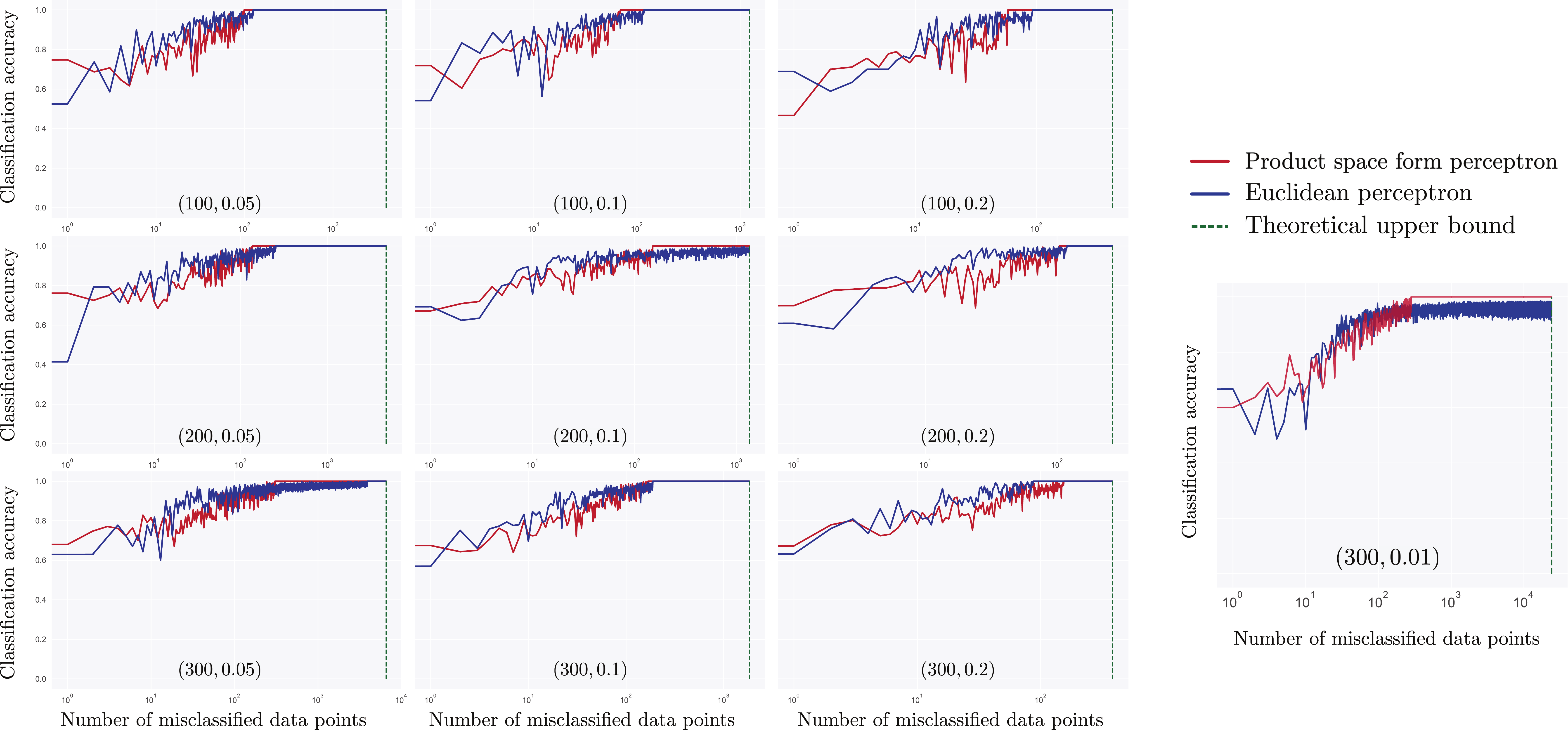}
  \caption{Classification accuracy (macro F1 scores) after each update of the Euclidean and product space form perceptron algorithms for different combinations of parameters $(N, \varepsilon)$.} 
  \label{fig:synthetic_perceptron}
\end{figure}

\subsection{Real-World Datasets}
In addition to synthetic datasets, we also examine real-world data and evaluate the practical performance of our product space form classifiers ---~\Cref{alg:mixed_curvature_perceptron,alg:mixed_curvature_svm} --- on such examples. In this part of the text, we focus on two scRNA-seq datasets and the CIFAR-100 dataset, with the detailed description of their properties and additional analysis of other datasets relegated to Appendix~\ref{sec:Additional_Real_World_Datasets}. All sensitive and privacy-revealing information has been removed from the datasets.

The datasets tested include:
\begin{enumerate}
\item Blood cell scRNA-seq datasets\footnote{\href{https://www.nature.com/articles/ncomms14049}{https://www.nature.com/articles/ncomms14049}} from~\citep{zheng2017massively}, including only information provided by $965$ landmark genes~\citep{subramanian2017next}. ``Landmark genes'' are genes that can be used to infer the activities of all other genes and are hence frequently used for scalable learning and genomic dimensionality reduction. In this case, we have $10$ classes with $94,655$ samples in total. 

\item Lymphoma\footnote{
\href{https://www.10xgenomics.com/resources/datasets/hodgkins-lymphoma-dissociated-tumor-targeted-immunology-panel-3-1-standard-4-0-0}{https://www.10xgenomics.com/resources/datasets/hodgkins-lymphoma-dissociated-tumor-targeted-immunology-panel-3-1-standard-4-0-0. Registration is required to access the content.}} and healthy donors\footnote{\href{https://www.10xgenomics.com/resources/datasets/pbm-cs-from-a-healthy-donor-targeted-compare-immunology-panel-3-1-standard-4-0-0}{https://www.10xgenomics.com/resources/datasets/pbm-cs-from-a-healthy-donor-targeted-compare-immunology-panel-3-1-standard-4-0-0. Registration is required to access the content.}} scRNA-seq binary-labeled datasets with $1020$-dimensional points and $13,410$ samples in total. 

\item CIFAR-100~\citep{krizhevsky2009learning}, an image dataset that contains $100$ classes of size $600$ each. Each image is of dimension $32 \times 32 \times 3$ (involving three colors).\footnote{\href{https://www.cs.toronto.edu/~kriz/cifar.html}{https://www.cs.toronto.edu/$\sim$kriz/cifar.html}}

\item Omniglot and MNIST data, for which the results are presented in~\Cref{sec:Additional_Real_World_Datasets}.
\end{enumerate}


To embed these datasets into different product space forms, we adapted and modified the mixed-curvature VAEs algorithm of~\citep{Skopek2020Mixed-curvature}. The original implementation of this algorithm does not allow the users to choose the number of layers and hidden dimensions of the network. So we introduce customized changes to make the approach suitable for use with datasets at different scales. For the Lymphoma dataset, we use two MLP layers with hidden dimension $200$ and train the network for $500$ epochs. For the blood cell landmark dataset, we use three MLP layers with hidden dimension $400$ and train the network for $200$ epochs. Other experimental setups are the same as the ones stated in~\citep{Skopek2020Mixed-curvature}.

We train linear classifiers on the low-dimensional (product) space form features extracted by the mixed-curvature VAEs. Hence, the performance of classification algorithms depends on the discriminative quality of the acquired features. Ideally, in order to maximize the classification performance, one may want to \emph{jointly} design and optimize the feature extraction and classification algorithms. However, joint optimization of feature extraction and classification objectives is technically challenging. We therefore decouple this process by first extracting low-dimensional features (via the previously described unsupervised mixed-curvature VAEs) and then train the linear classifiers.

\emph{Perceptron}: We split the datasets into $60\%$ training and $40\%$ test point sets. In general, embedded datasets are not linearly separable. For a fair comparison, we allow all perceptron algorithms to go over the whole dataset only once to simulate the online learning scenario. For datasets with two classes (Lymphoma and healthy donor dataset), we perform binary classifications on the whole dataset with $30$ different splits for training and testing sets. For datasets with multiple classes (CIFAR-100 and blood cell landmark gene datasets), we perform binary classifications on samples from two randomly chosen classes and repeat it three times with different splits for training and testing sets. For the CIFAR-100 dataset, we choose $30$ class pairs, whereas for the blood cell landmark gene dataset, we choose all possible pairs, i.e., a total of $\binom{10}{2}=45$ pairs. The mean and $95\%$ confidence interval of the obtained macro F-1 scores are reported in~\Cref{fig:numerical_results} (bottom row).

\emph{SVM}: For simplicity, we set $\alpha_{\bE}= \alpha_{\bS}=\alpha_{\bH}=1,$ and relax the optimization problem by leaving out nonconvex constraints. To lower the computational complexity, for each dataset, we also only use $150$ training samples and reserve the remaining ones for testing. Other experimental settings are identical to those of the perceptron.

\begin{figure}[t]
  \center
  \includegraphics[width=1 \linewidth]{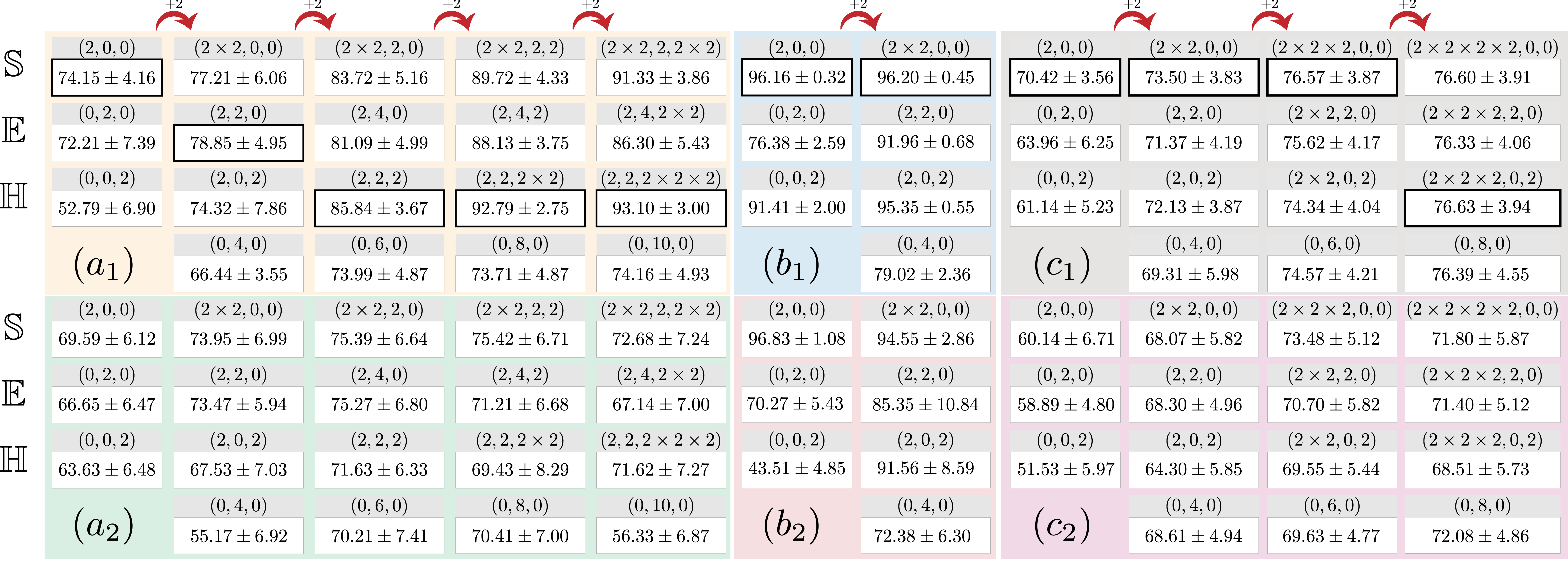}
  \caption{The average classification accuracy $(\%)$ and $95\%$-confidence interval for the SVM (top row) and the perceptron (bottom row) algorithms in product space form with different signatures. Datasets are blood cell landmark gene expressions (Figures $(a_1)$ and $(a_2)$), Lymphoma and healthy donors expressions (Figures $(b_1)$ and $(b_2)$), and CIFAR-100 (Figures $(c_1)$ and $(c_2)$). The choice of the embedding space is determined by the signature estimation method discussed in~\Cref{sec:signature_estimation}, and the resulting signatures are in shown bold rectangles.}
  \label{fig:numerical_results}
\end{figure}

The results across different datasets, learning methods, and embedding signatures (i.e., choices of dimensions of the components in the product spaces) suggest that product spaces offer significantly better low-dimensional representations for complex data structures, especially for scRNA-seq data; see \Cref{fig:numerical_results}. Generally, a \emph{higher-dimensional signature should lead to a better classification accuracy}. However, the performance of the classification method crucially depends on the discriminative quality of the features extracted from the mixed-curvature VAEs --- which are not guaranteed to return embeddings of accuracy that increases with the dimension of the ambient space (provided that other hyperparameters are fixed). 

Furthermore, finding an analytic expression for a signature that allows for near-optimal embedding distortion is a hard problem that requires a sophisticated analysis of the geometry of datasets, and is thus beyond the scope of this work. Nevertheless, the signature estimation heuristic, introduced in~\Cref{sec:signature_estimation}, lets use narrow down the choices for suitable signatures and lends itself to a process to progressively improves the classification results. In~ \Cref{fig:numerical_results} $(a_1)$, $(b_1)$, and $(c_1)$, we used the SVM classification results to estimate the near-optimal signatures for each of the three datasets. We then used these signatures for our perceptron experiments; see~\Cref{fig:numerical_results} $(a_2)$, $(b_2)$, and $(c_2)$. The improvements in classification accuracy for the CIFAR-100 data are modest ($\sim2\%$) but the performance of product space classifiers on the scRNA-seq datasets offers an average of $\sim15\%$ improvements compared to purely Euclidean classifiers. This is not surprising given the prior biological insight that populations of cells in a tissue follow hierarchical evolutionary trajectories (suitably 
captured by hyperbolic spaces) and \say{cyclic} cell-cycle phases (suitably captured by spherical spaces).

\acks{The authors would like to thank Prof. Ivan Dokmani\'c for helpful discussions and suggestions.}

\appendix
\section{Proof of \Cref{thm:mixed_vc}}\label{sec:mixed_vc_proof}
We let $x_1, \ldots, x_N \in \mathcal{M}$, where $\mathcal{M} = \R^{d_{\bE}} \times \bS^{d_{\bS}} \times \bH^{d_{\bH}}$ and $N = d_{\bE}+d_{\bS}+d_{\bH}+1$, such that
\[
x_n = \left\{\begin{aligned}\quad
   & (x_{\bE,n}, x_{\bS,N}, x_{\bH,N})^{\T}  & \mbox{ for } & n \in \set{1, \ldots, d_{\bE}}   \\
   & (x_{\bE,N}, x_{\bS,n}, x_{\bH,N})^{\T}  & \mbox{ for } & n \in \set{d_{\bE}+1, \ldots, d_{\bE}+d_{\bS}}   \\
   & (x_{\bE,N}, x_{\bS,N}, x_{\bH,n})^{\T}  & \mbox{ for } & n \in \set{d_{\bE}+d_{\bS}+1, \ldots, d_{\bE}+d_{\bS}+d_{\bH}} \\
   & (x_{\bE,N}, x_{\bS,N}, x_{\bH,N})^{\T}  & \mbox{ for } & n = d_{\bE}+d_{\bS}+d_{\bH}+1
     \end{aligned}\right.
\]
where $x_{\bE, N}, x_{\bS, N},$ and $x_{\bH, N}$ are three arbitrary points in $\R^{d_{\bE}}$, $\bS^{d_{\bS}}$, and $\bH^{d_{\bH}}$. Moreover, we restrict $\set{x_{\bE, n}}_{n = 1}^{d_{\bE}}$, $\set{x_{\bS, n}}_{n = 1}^{d_{\bS}}$, and $\set{x_{\bH, n}}_{n = 1}^{d_{\bH}}$ such that --- together with $x_{\bE, N}$, $x_{\bS, N}$, and $x_{\bH, N}$ ---  can be shattered with linear classifiers in $\R^{d_{\bE}}$, $\bS^{d_{\bS}}$, and $\bH^{d_{\bH}}$. Consequently, for any vector $t = (t_n)_{n=1}^{N-1} \in \R^{N-1}$, we can find Euclidean, spherical and hyperbolic linear classifiers --- with parameters $(w_{\bE}, b), w_{\bS}$, and $w_{\bH}$ --- such that the following conditions hold true:
\begin{align}
&\forall n \in \set{1, \ldots, d_{\bE} }:  & w_{\bE}^{\T} x_{\bE, n}+ b &= t_n; &  w_{\bE}^{\T} x_{\bE, N}+ b &= 0. \label{eq:vc_euclidean}\\
& \forall n \in \set{d_{\bE}+1, \ldots, d_{\bE}+d_{\bS}}: & \langle w_{\bS}, x_{\bS, n} \rangle &= \sin (t_n); & \langle w_{\bS}, x_{\bS, N}\rangle &= 0.  \label{eq:vc_spherical}\\
& \forall n \in \set{d_{\bE}+d_{\bS}+1, \ldots, d_{\bE}+d_{\bS}+d_{\bH}}: & [ w_{\bH}, x_{\bH, n}] &= \mathrm{sinh}(t_n); &  [w_{\bH}, x_{\bH, N}] &= 0.  \label{eq:vc_hyperbolic}
\end{align}
Now, let $w = (w_{\bE}, b, w_{\bS}, w_{\bH})$ be the parameter vector for a linear classifier in product space form $\mathcal{M}$. Then, we have
\[
l^{\mathcal{M}}_{w}(x_n) = \left\{\begin{aligned}\quad
   & \mathrm{sgn} ( w_{\bE}^{\T} x_{\bE, n}+ b)  & \mbox{ for } & n \in \set{1, \ldots, d_{\bE}}   \\
   & \mathrm{sgn} \big( \mathrm{asin} (  \langle w_{\bS}, x_{\bS, n}\rangle ) \big)  & \mbox{ for } & n \in \set{d_{\bE}+1, \ldots, d_{\bE}+d_{\bS}}   \\
   & \mathrm{sgn} \big( \mathrm{asinh} ( [ w_{\bH}, x_{\bH, n}]) \big) & \mbox{ for } & n \in \set{d_{\bE}+d_{\bS}+1, \ldots, d_{\bE}+d_{\bS}+d_{\bH}},
     \end{aligned}\right.
\]
or $l^{\mathcal{M}}_{w}(x_n)= \mathrm{sign}(t_n)$, for all $n \in [N-1]$. This is the direct result of constructing the point set according to conditions described in equations \eqref{eq:vc_euclidean}, \eqref{eq:vc_spherical}, and \eqref{eq:vc_hyperbolic}.
This means $\set{ x_n }_{n=1}^{N-1}$ can be shattered with linear classifiers in $\mathcal{M}$. Now, we want to show that the point $x_N = (x_{\bE,N}, x_{\bS,N}, x_{\bH,N})^{\T}$ can also be shattered.
\begin{lemma}
Let $w_{\varepsilon} = w + \varepsilon t_{N} (x_{\bE,N},1, x_{\bS,N}, -x_{\bH,N} )^{\T}$ for a positive scalar $\varepsilon$ and an arbitrary scalar $t_N \in \R$.\footnote{Note that \say{scaling} the components of the vector $w$ does not change the labels produced by $l_w$. However, it will violate the norm constraints required to define the distance-based classifiers. So, we can re-scale the components of $w$ such that  $w_{\varepsilon} = w + \varepsilon t_{N} (x_{\bE,N},1, x_{\bS,N}, -x_{\bH,N} )^{\T}$ complies with the required norm constraints for each Euclidean, spherical, and hyperbolic components, i.e., $w_{\bE, \varepsilon} = \frac{1}{\norm{ w_{\bE} + \varepsilon t_{N} x_{\bE,N}}} (w_{\bE} + \varepsilon t_{N} x_{\bE,N})$, $w_{\bS, \varepsilon} = \frac{1}{\norm{ w_{\bS} + \varepsilon t_{N} x_{\bS,N}}}(w_{\bS} + \varepsilon t_{N} x_{\bS,N})$, and $w_{\bH, \varepsilon} = \frac{1}{\sqrt{ -[w_{\bH} - \varepsilon t_{N} x_{\bH,N}, w_{\bH} - \varepsilon t_{N} x_{\bH,N}]}}(w_{\bH} - \varepsilon t_{N} x_{\bH,N})$. For vanishing $\varepsilon \rightarrow 0$, we have $w_{\bE, \varepsilon} \rightarrow w_{\bE}$, $w_{\bS, \varepsilon} \rightarrow w_{\bS}$, and $w_{\bH, \varepsilon} \rightarrow w_{\bH}$. However, we do not specifically reflect these normalizations in our proof since we are only interested labels produced by $l_{w_{\varepsilon}}$ when $\varepsilon \rightarrow 0$. } This perturbed classifier computes the following labels,
\[
l^{\mathcal{M}}_{ w_{\varepsilon} }(x_n) = \left\{\begin{aligned}\quad
   & \mathrm{sgn} ( w_{\bE}^{\T} x_{\bE, n}+ b + c_{1,n} \varepsilon +o(\varepsilon) )   & \mbox{ for } & n \in \set{1, \ldots, d_{\bE}}   \\
   & \mathrm{sgn} ( \langle w_{\bS}, x_{\bS, n}\rangle + c_{2,n} \varepsilon +o(\varepsilon))  & \mbox{ for } & n \in \set{d_{\bE}+1, \ldots, d_{\bE}+d_{\bS}}   \\
   & \mathrm{sgn} ( [ w_{\bH}, x_{\bH, n}] + c_{3,n} \varepsilon +o(\varepsilon))  & \mbox{ for } & n \in \set{d_{\bE}+d_{\bS}+1, \ldots, d_{\bE}+d_{\bS}+d_{\bH}} \\
      & \mathrm{sign} ( t_n	 c_4  + O(\varepsilon) )  & \mbox{ if } & n = d_{\bE}+d_{\bS}+d_{\bH}+1,
     \end{aligned}\right.
\]
where $c_{1,n} = (3+x_{\bE,0}^{\T} x_{\bE, n})t_N$, $c_{2,n} = \big( \norm{x_{\bE, 0}}_2^2+2+\langle x_{\bS,N}, x_{\bS,n}\rangle \mathrm{asin}^{\prime}( \langle w_{\bS}, x_{\bS,n}\rangle ) \big)t_N$, $c_{3,n} =  \norm{x_{\bE, 0}}_2^2+2-\mathrm{asinh}^{\prime} ( [ w_{\bH}, x_{\bH,n} ] ) [ x_{\bH,N}, x_{\bH,n} ] $, and
where $c_4 =  \norm{x_{\bE, 0}}_2^2+3 > 0$. This result proves that $l^{\mathcal{M}}_{ w_{\varepsilon} }(x_n) = \mathrm{sgn}(t_n)$ for all $n \in [N]$, when $\varepsilon \rightarrow 0$.
\end{lemma}
\begin{proof}
Let $n \in \set{1, \ldots, d_{\bE}}$. Then, we have
\begin{align*}
&l^{\mathcal{M}}_{ w_{\varepsilon} }(x_n) \\
&= \mathrm{sgn} \big( w_{\bE}^{\T} x_{\bE, n}+ b+\varepsilon t_N(1+x_{\bE,N}^{\T} x_{\bE, n}) + \mathrm{asin}( \varepsilon t_N\langle x_{\bS,N}, x_{\bS,N}\rangle) + \mathrm{asinh}( \varepsilon t_N[-x_{\bH,N} , x_{\bH,N}]) \big) \\
&\stackrel{\text{(a)}}{=}  \mathrm{sgn} \big( w_{\bE}^{\T} x_{\bE, n}+ b+\varepsilon t_N(1+x_{\bE,N}^{\T} x_{\bE, n}) + \mathrm{asin}( \varepsilon t_N ) + \mathrm{asinh}( \varepsilon t_N) \big)\\
&\stackrel{\text{(b)}}{=}  \mathrm{s
gn} \big( w_{\bE}^{\T} x_{\bE, n}+ b+\varepsilon t_N(1+x_{\bE,N}^{\T} x_{\bE, n}) + 2\epsilon t_N+ o(\varepsilon) \big) \\
&=  \mathrm{sign} \big( w_{\bE}^{\T} x_{\bE, n}+ b+\varepsilon t_N(3+x_{\bE,N}^{\T} x_{\bE, n}) + o(\varepsilon) \big) .
\end{align*}
where $\text{(a)}$ is due to the norm constraints for points in hyperbolic and spherical spaces ($\langle x, x \rangle = 1$ if $x \in  \bS^d$, and $[x,x] = -1$ if $x \in \bH^d$), and $\text{(b)}$ is due to the facts that $\mathrm{asinh}(x) = x + o(x)$ and $\mathrm{asinh}(x) = x + o(x)$ for $x \rightarrow 0$. This give us $c_{1,n} = (3+x_{\bE,N}^{\T} x_{\bE, n})t_N$.

For $n \in \set{d_{\bE}+d_{\bS}+1, \ldots, d_{\bE}+d_{\bS}+d_{\bH}}$, we have
\begin{align*}
l^{\mathcal{M}}_{ w_{\varepsilon} }(x_n) &= \mathrm{sgn} \big( (\varepsilon t_N x_{\bE, N})^{\T} x_{\bE, N}+\varepsilon t_N+ \mathrm{asin}( \langle w_{\bS}, x_{\bS,n}\rangle + \varepsilon t_N \langle x_{\bS,N}, x_{\bS,n}\rangle) + \mathrm{asinh}( \varepsilon t_N) \big) \\
&\stackrel{\text{(a)}}{=} \mathrm{sgn} \big( \varepsilon t_N ( \norm{x_{\bE, N}}_2^2+2)+ \mathrm{asin}( \langle w_{\bS}, x_{\bS,n}\rangle ) + \varepsilon t_N \langle x_{\bS,N}, x_{\bS,n}\rangle \mathrm{asin}^{\prime}( \langle w_{\bS}, x_{\bS,n}\rangle)+o(\varepsilon)  \big) \\
&= \mathrm{sgn} \big( \mathrm{asin}( \langle w_{\bS}, x_{\bS,n}\rangle )  +\varepsilon t_N ( \norm{x_{\bE, N}}_2^2+2+\langle x_{\bS,N}, x_{\bS,n}\rangle \mathrm{asin}^{\prime}( \langle w_{\bS}, x_{\bS,n}\rangle))+o(\varepsilon) \big)
\end{align*}
where $\text{(a)}$ is due to the first order Taylor approximation of $\mathrm{asin}(\cdot)$ function. This gives $c_{2,n} = \big( \norm{x_{\bE, N}}_2^2+2+\langle x_{\bS,N}, x_{\bS,n}\rangle \mathrm{asin}^{\prime}( \langle w_{\bS}, x_{\bS,n}\rangle ) \big)t_N$.

For $n \in \set{d_{\bE}+d_{\bS}+1, \ldots, d_{\bE}+d_{\bS}+d_{\bH}}$, we have
\begin{align*}
&l^{\mathcal{M}}_{ w_{\varepsilon} }(x_n) \\
&= \mathrm{sign} \big( \varepsilon t_N ( \norm{x_{\bE, N}}_2^2+1)+ \mathrm{asin}( \varepsilon t_N  )+ \mathrm{asinh}( [ w_{\bH}, x_{\bH,n} ] - \varepsilon t_N [ x_{\bH,N}, x_{\bH,n} ] ) \big)  \\
&\stackrel{\text{(a)}}{=}  \mathrm{sgn} \big( \varepsilon t_N ( \norm{x_{\bE, N}}_2^2+2) + \mathrm{asinh}( [ w_{\bH}, x_{\bH,n} ] )-  \varepsilon t_N  \mathrm{asinh}^{\prime} ( [ w_{\bH}, x_{\bH,n} ] ) [ x_{\bH,N}, x_{\bH,n} ] + o(\varepsilon ) \big) \\
&= \mathrm{sgn} \big( \mathrm{asinh}( [ w_{\bH}, x_{\bH,n} ] ) + \varepsilon t_N ( \norm{x_{\bE, N}}_2^2+2-\mathrm{asinh}^{\prime} ( [ w_{\bH}, x_{\bH,n} ] ) [ x_{\bH,N}, x_{\bH,n} ] ) + o(\varepsilon)  \big)
\end{align*}
where $\text{(a)}$ is due to the first order Taylor approximation of $\mathrm{asinh}(\cdot)$ function.  This gives us $c_{3,n} =  \big( \norm{x_{\bE, N}}_2^2+2-\mathrm{asinh}^{\prime} ( [ w_{\bH}, x_{\bH,n} ] ) [ x_{\bH,N}, x_{\bH,n} ] \big) t_N$.

Finally, let $n = d_{\bE}+d_{\bS}+d_{\bH}+1$. Then, we have
\begin{align*}
l^{\mathcal{M}}_{ w_{\varepsilon} }(x_N) &= \mathrm{sgn} (\varepsilon t_N x_{\bE, N})^{\T} x_{\bE, N}+\varepsilon t_N + \mathrm{asin} ( \varepsilon t_N \langle x_{\bS, N}, x_{\bS, N} \rangle ) + \mathrm{asinh} (-\varepsilon t_N [ x_{\bH, N}, x_{\bH, N} ] ) \\
&=  \mathrm{sgn} ( \varepsilon t_N ( \norm{x_{\bE, N}}_2^2+1)  +\varepsilon t_N + \varepsilon t_N +o(\varepsilon) ) \\
&= \mathrm{sgn} ( t_N ( \norm{x_{\bE, N}}_2^2+3)  +O(\varepsilon) ).
\end{align*}
Therefore, $c_4 = \norm{x_{\bE, N}}_2^2+3 > 0$.
\end{proof}
This lemma directly shows that linear classifiers in $\mathcal{M}$ can shatter at least $\mathrm{dim}(\mathcal{M})+1$ points. This argument can be extended to general product space forms.

\section{Proof of \Cref{thm:mixed_curvature_perceptron}}\label{sec:perceptron_proof}
Let $x_1, \ldots, x_N \in \mathcal{M} = \bE^{d_{\bS}} \times \bS^{d_{\bS}} \times \bH^{d_{\bH}}$ and $R$ be an upper bound for the norm of the hyperbolic component of points, i.e., $\norm{x_{\bH,n}}_2 \leq R$ for all $n \in [N]$. The linear classifier in product space form can be written as
\begin{align*}
	l^{\mathcal{M}}_w(x_n) &= \mathrm{sgn} \big( w_{\bE}^\T x_{\bE,n} + b + \alpha_{\bS}\mathrm{asin}(w_{\bS}^\T x_{\bS,n}) +\alpha_{\bH} \mathrm{asinh}((Rw_{\bH})^\T  \frac{1}{R}H x_{\bH,n}) \big) \\
	&=  \mathrm{sgn}  \big( \langle \phi(w), M\psi(x_n) \rangle_{\mathcal{H}} \big),
\end{align*}
where $\mathcal{H}$ is the product of $\R^{d_{\bE}+1}$, $\mathcal{H}_{d_{\bS}}$ and $\mathcal{H}_{d_{\bH}}$ accompanied by their corresponding inner products, $M = \mathrm{diag}\set{I,I,M_{d_{\bH}}}$ is an operator on $\mathcal{H}$, and
\begin{align*}
\phi(w) &= \big( b, w_{\bE}, \sqrt{\alpha_{\bS}} \phi_{d_{\bS}}(\frac{1}{\sqrt{C_{\bS}}}w_{\bS}), \sqrt{\alpha_{\bH}} \phi_{d_{\bH}}(R w_{\mathbb{H}}) \big) \in \mathcal{H}, \\
\psi(x_n) &= \big( 1, x_{\bE,n}, \sqrt{\alpha_{\bS}}\phi_{d_{\bS}}(\sqrt{C_{\bS}}x_{\bS,n}),\sqrt{\alpha_{\bH}} \phi_{d_{\bH}}(H\frac{1}{R}x_{\bH,n}) \big),
\end{align*}
where $\phi_d(\cdot)$ is defined as in the proof of \Cref{lem:taylor_series_kernel}. Therefore, we have
\begin{align*}
&\langle \phi(w), M \psi(x)\rangle_{\mathcal{H}} \\
&=w_{\bE}^\T x_{\bE} + b + \alpha_{\bS}\langle \phi_{d_{\bS}}(\frac{1}{\sqrt{C_{\bS}}}w_{\bS}), \phi_{d_{\bS}}(\sqrt{ C_{\bS} }x_{\bS}) \rangle_{\mathcal{H}_{d_{\bS}}} + \alpha_{\bH} \langle \phi_{d_{\bH}}(R w_{\bH}), M_{d_{\bH}} \phi_{d_{\bH}}(\frac{1}{R} Hx_{\bH}) \rangle_{\mathcal{H}_{d_{\bH}}}.
\end{align*}
From the problem's assumptions, the data points are linearly separable, i.e.,
\[
    \forall n \in [N]: y_n \langle \phi(w^*), M \psi(x_n) \rangle_{\mathcal{H}}  \geq \varepsilon, 
\]
for a specific parameter $w^*$. Similar to the hyperbolic perceptron setting, we use the following update rule in RKHS
\[
\phi_{\mathcal{H}}^{k+1} = \phi_{\mathcal{H}}^{k} + y_n  M \psi(x_n) \ \mbox{ if } \ y_n \langle \phi_{\mathcal{H}}^{k}, M  \psi(x_n) \rangle_{\mathcal{H}} \leq 0.
\]
If we initialize $\phi_{\mathcal{H}}^0 = 0 \in \mathcal{H}$, we have
\begin{align*}
    \langle \phi(w^*), \phi_{\mathcal{H}}^{k+1} \rangle_{\mathcal{H}} &= \langle \phi(w^*),  \phi_{\mathcal{H}}^{k} \rangle_{\mathcal{H}} + \langle \phi(w^*), M y_n \psi(x_n) \rangle_{\mathcal{H}} \\
    &\geq  \langle \phi(w^*),  \phi_{\mathcal{H}}^{k} \rangle_{\mathcal{H}} + \varepsilon \\
    &\geq k \varepsilon.
\end{align*}
On the other hand, we can bound the norm as
\begin{align*}
    &\langle \phi_{\mathcal{H}}^{k+1}, \phi_{\mathcal{H}}^{k+1} \rangle_{\mathcal{H}} \\ &=  \langle \phi_{\mathcal{H}}^{k}, \phi_{\mathcal{H}}^{k} \rangle_{\mathcal{H}} +\langle y_n  M\psi(x_n), y_n  M\psi(x_n) \rangle_{\mathcal{H}} + 2 \langle \phi_{\mathcal{H}}^{k} ,  y_n  M\psi(x_n)\rangle_{\mathcal{H}} \\
    &\leq \langle \phi_{\mathcal{H}}^{k}, \phi_{\mathcal{H}}^{k} \rangle_{\mathcal{H}} + \langle \psi(x_n), \psi(x_n) \rangle_{\mathcal{H}} \\
    &\leq \langle \phi_{\mathcal{H}}^{k}, \phi_{\mathcal{H}}^{k} \rangle_{\mathcal{H}}  + 1 + \norm{x_{\bE,n}}_2^2 + \alpha_{\bS} \langle \phi_{d_{\bS}}(\sqrt{C_{\bS}}x_{\bS,n}), \phi_{d_{\bS}}(\sqrt{C_{\bS}}x_{\bS,n}) \rangle_{\mathcal{H}_{d_{\bS}}} \\
    &+ \alpha_{\bH} \langle \phi_{d_{\bH}}(\frac{1}{R}Hx_{\bH,n}),  \phi_{d_{\bH}}( \frac{1}{R}Hx_{\bH,n}) \rangle_{\mathcal{H}_{d_{\bH}}} \\
    &\stackrel{\mathrm{(a)}}{\leq}  k (1+R_{\bE}^2 + (\alpha_{\bS}+\alpha_{\bH})\frac{\pi}{2}),
\end{align*}
where $R_{\bE}$ is an upper bound for the norm of the Euclidean components of the vectors, and $\mathrm{(a)}$ is due to
\[
	  \langle \phi_{d_{\bS}}(\sqrt{C_{\bS}}x_{\bS,n}), \phi_{d_{\bS}}(\sqrt{C_{\bS}}x_{\bS,n}) \rangle_{\mathcal{H}_{d_{\bS}}}  = \mathrm{asin} ( C_{\bS} x_{\bS,n}^\T x_{\bS,n}) = \frac{\pi}{2},
\]
and 
\[
	  \langle \phi_{d_{\bH}}(\frac{1}{R}Hx_{\bH,n}),  \phi_{d_{\bH}}( \frac{1}{R}Hx_{\bH,n}) \rangle_{\mathcal{H}_{d_{\bH}}} = \mathrm{asin} ( \frac{1}{R^2}x_{\bH,n}^\T x_{\bH,n}) \leq \frac{\pi}{2}.
\]
Hence, 
\begin{align*}
	\frac{( \langle \phi_{\mathcal{H}}^{k+1}, \phi(w^*) \rangle_{\mathcal{H}} )^2}{\langle \phi_{\mathcal{H}}^{k+1}, \phi_{\mathcal{H}}^{k+1} \rangle_{\mathcal{H}} \langle \phi(w^*), \phi(w^*) \rangle_{\mathcal{H}}} &\geq \frac{k^2 \varepsilon^2}{ kB \langle \phi(w^*), \phi(w^*) \rangle_{\mathcal{H}} } \\
	&= k \frac{\varepsilon^2}{ B  \langle \phi(w^*), \phi(w^*) \rangle_{\mathcal{H}} },
\end{align*}
where $B = 1+R_{\bE}^2 +  (\alpha_{\bS}+\alpha_{\bH})\frac{\pi}{2}$. Therefore, convergence is guaranteed in $k \leq \frac{B \langle \phi(w^*), \phi(w^*) \rangle_{\mathcal{H}}}{\varepsilon^2}$ steps. Finally, the upper bound for the $\ell_2$ norm of $w_{\bH}$ guarantees the boundedness of $\langle \phi(w^*), \phi(w^*) \rangle_{\mathcal{H}}$.

\section{Proof of \Cref{thm:hyperbolic_perceptron}} \label{sec:hyperbolic_perceptron_proof}
Let $w^0 = 0 \in \R^{d+1}$ and let $w^{k} \in \R^{d+1}$ be the estimated normal vector at the $k$-th iteration of the perceptron algorithm (see~\Cref{alg:hyperbolic_perceptron}). If the point $x_n \in \bH^d$ ($y_n [w^{k} , x_n] < 0$) is missclassified, 
the perceptron algorithm produces the $(k+1)$-th estimate of the normal vector according to
\[
w^{k+1} = w^{k} + y_n H x_n.
\]
Let $w^*$ be the normal vector that classifies all the points with margin of at least $\varepsilon$, i.e., $y_n\, \mathrm{asinh} ( [w^*,x_n] ) \geq \varepsilon,$ $\forall \, n \in [N]$, and $[w^*, w^*] = 1$.
Then, we have
\begin{align*}
	(w^*)^\T w_{k+1} &= (w^*)^\T w^{k} + y_n [w^*, x_n] \\
	& \geq (w^*)^\T w^{k} +  \mathrm{sinh}(\varepsilon )\\
	& \geq  k  \mathrm{sinh}(\varepsilon ).
\end{align*}
In what follows, we provide an upper bound on $\norm{w^{k+1}}_2$,
\begin{align*}
    \norm{w^{k+1}}_2^2 &= \norm{w^{k} +  y_n H x_n}_2^2 \\
    &=\norm{w^{k}}_2^2+ \norm{x_n}_2^2 + 2 y_n [w^{k}, x_n] \\
    &\stackrel{\mathrm{(a)}}{\leq} \norm{w^{k}}_2^2 +  R^2 \\
    &=  k  R^2,
\end{align*}
where $\mathrm{(a)}$ is due to $ \norm{x_n}_2^2 \leq R^2$ and $y_n [w^{k}, x_n]  \leq 0$, due to the error in classifying the point $x_n$. Hence,
\begin{equation} \label{eq:norm_inequality}
		\norm{w^{k+1}}_2 \leq \sqrt{k} R \ \mbox{ and } \ (w^*)^\T w^{k+1}  \geq  k  \mathrm{sinh}(\varepsilon ).
\end{equation}
To complete the proof, define $\theta_k = \mathrm{acos} ( \frac{(w^k)^\T w^*}{\norm{w^k}_2\norm{w^*}_2})$. Then, 
\begin{align*}
	\ \frac{(w^{k+1})^\T w^*}{\norm{w^{k+1}}_2\norm{w^*}_2} &\stackrel{\mathrm{(a)}}{\geq}  \frac{ k \mathrm{sinh}(\varepsilon)}{ \sqrt{k}R\norm{w^*}_2} \\
	&= \sqrt{k}\frac{ \mathrm{sinh}(\varepsilon)}{R\norm{w^*}_2}, 
\end{align*}
where $\mathrm{(a)}$ follows from \eqref{eq:norm_inequality}. For $k \geq \Big( \frac{R \norm{w^*}_2}{\mathrm{sinh}( \varepsilon)} \Big)^2$, we have $w^{k+1} = \alpha_{k+1} w^*$ for a positive scalar $\alpha_{k+1} $. Hence,  $\frac{1}{\sqrt{[w^{k+1},w^{k+1}]}}w^{k+1} = w^*$.
\begin{algorithm}[b]
   \caption{Hyperbolic Perceptron} \label{alg:hyperbolic_perceptron}
\begin{algorithmic}
	\STATE {\bfseries Input:} $\set{x_n, y_n}_{n=1}^{N}$ : a set of point-labels in $\bH^{d_{\bH}} \times \set{-1, 1}$.
	\STATE {\bfseries Initialization:} $w^0 = 0 \in \R^{d_{\bH}+1}$, $k = 0$, $n = 1$.
	\REPEAT
	\IF{$ \mathrm{sgn} ( [w^k,x_n] ) \neq y_n $}
	\STATE $w^{k+1} = w^k + y_n H x_n$;
	\STATE $k = k+1$;
	\ENDIF
	\STATE $n = \mathrm{mod}(n, N)+1$;
   	\UNTIL{Convergence criteria is met.}
\end{algorithmic}
\end{algorithm}

\subsection{Discussion} \label{sec:further_discussions_hyperbolic}
 A purely hyperbolic perceptron (in the 'Loid model used in this work) was described in~\citep{weber2020robust}. The proposed update rules read as
\begin{align}
    & u^{k} = w^{k} + y_{n} x_{n} \ \mbox{ if }  -y_n [w^{k}, x_n] < 0 \label{eq:WebP_eq1}\\
    & w^{k+1} = u^k/ \min\{1,\sqrt{[u^{k}, u^k]}\}\label{eq:WebP_eq2},
\end{align}
where~\eqref{eq:WebP_eq2} is a \say{normalization step}. Unfortunately, the above update rule does not allow the hyperbolic perceptron algorithm (\Cref{eq:WebP_eq1,eq:WebP_eq2}) to converge, which is due to the choice of the update direction. The convergence issue is also illustrated by the following two examples. 

Let $x_1 = [\sqrt{2}, 1, 0]^\T  \in \bH^2$ with label $y_1 = 1$. We choose the initial vector in the update rule to be $w^{0} = [\frac{-\sqrt{2} + 3}{4}, \frac{-1+ 3\sqrt{2}}{4}, 0]^\T$ (in contrast to $w^{0} = e_2$, which was chosen in the proof~\citep{weber2020robust}). This is a valid choice because $[w^0,w^0] = \frac{1}{2} > 0$.
In the first iteration, we must hence update $w^{0}$ since $-y_1[w^{0},x_1] = -\frac{1}{4} < 0$. From~\eqref{eq:WebP_eq1}, we have $u^{0} = w^0+y_1 x_1 = [\frac{3\sqrt{2} + 3}{4}, \frac{3+ 3\sqrt{2}}{4}, 0]^\T$, and $[u^0,u^0] = 0$. This means that $w^1 = \frac{1}{\sqrt{[u^0, u^0]}} u^0$ is clearly ill-defined. 

As another example, let  $w^\star$ be the optimal vector with which we can classify all data points with margin $\varepsilon$. If we simply choose $w^0 = 0$, then we can satisfy the required condition $[w^0,w^\star] \geq 0$ postulated for the hyperbolic perceptron in~\cite{weber2020robust}. This leads to $u^0 = x_1$. Then, for any $x_1 \in \bH^2$, we have $[x_1,x_1] = -1$, which results in a normalization factor $\sqrt{[u^0,u^0]}$ that is a complex number.

\subsection{Simulated Convergence Analysis of The New Hyperbolic Perceptron}\label{sec:Convergence_Analysis_Hyperbolic_Perceptron}
As pointed out in~\Cref{sec:further_discussions_hyperbolic}, the hyperbolic perceptron described in~\citep{weber2020robust} does not converge. This fact can be easily observed through simulations and the two previously provided counterexamples why this may be the case. We report the experimental results to validate this point and to suggest using the newly developed perceptron, and in addition, to demonstrate that a convergence rate of $O\left(\frac{1}{\sinh(\varepsilon)}\right)$ is not possible.

First, we randomly generate a $w^*$ such that $[w^*,w^*]=1$. Then, we generate a random set of $N=5,000$ points $\{x_i\}_{i=1}^N$ in $\bH^2$. For margin values $\varepsilon \in [0.1, 1]$, we remove points that violate the required distance to the classifier (parameterized by $w^*$), i.e., we decimate the points so that the condition $\forall n: |[w^*, x_n]|\geq \sinh(\varepsilon)$ is satisfied. Then, we assign binary labels to each data point according to the optimal classifier so that $y_n = \mathrm{sgn} \big( \mathrm{asinh}( [w^*, x_n] ) \big)$. We repeat this process for $100$ different values of $\varepsilon$.

In the first experiment, we compare the performance of our new hyperbolic perceptron~\Cref{alg:hyperbolic_perceptron} to the one described in Algorithm 1 of~\citep{weber2020robust} by running both until the number of updates meets a preset upper bound (stated in Theorem 3) or until the classifier correctly classifies all data points. In~\Cref{fig:weber_not_converge} $(a)$, we report the classification accuracy of each method on the training data. Note that our theoretically established convergence rate $O\left(\frac{1}{\sinh^2(\varepsilon)}\right)$ is larger than $O\left(\frac{1}{\sinh(\varepsilon)}\right),$ the rate derived in Theorem 3.1 of~\citep{weber2020robust}. So, for the second experiment, we repeated the same process but terminated both algorithms after $O\left(\frac{1}{\sinh(\varepsilon)}\right)$ updates. The classification performance of the two in this setting is shown in~\Cref{fig:weber_not_converge} $(b)$. From these results, one can easily conclude that $(1)$ our algorithm  always converge within the theoretical upper bound provided in Theorem 3, and $(2)$ both methods violate the theoretical convergence rate upper bound of~\citep{weber2020robust}. 

\begin{figure}[t]
  \center
  \includegraphics[width=1 \linewidth]{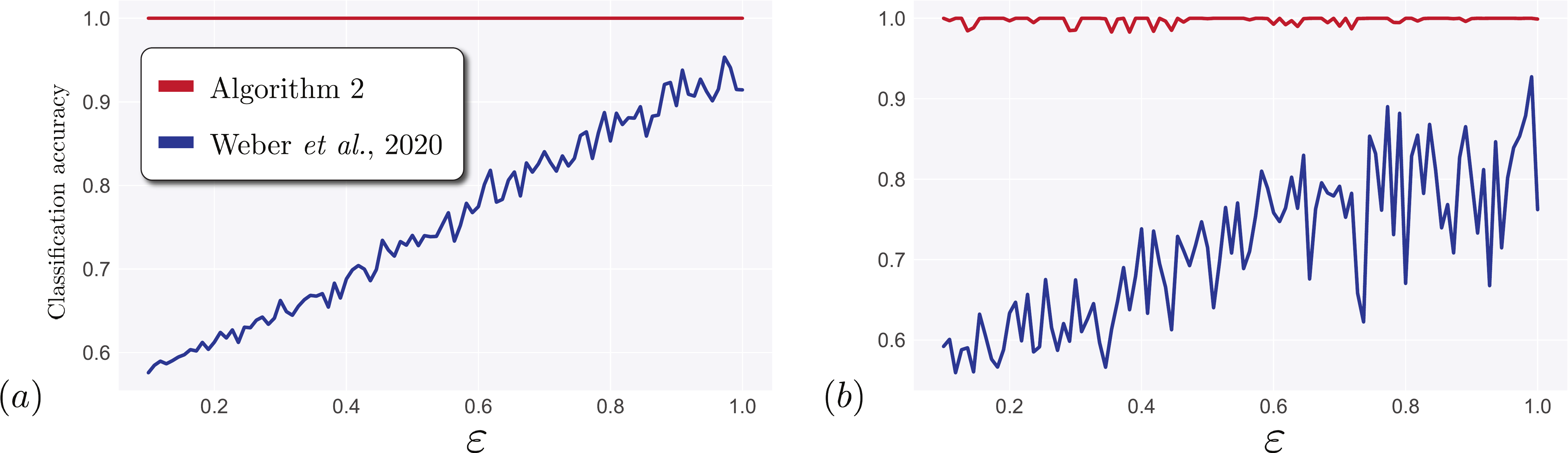}
  \caption{A comparison between the classification accuracy of our new hyperbolic perceptron~\Cref{alg:hyperbolic_perceptron} and the algorithm in~\citep{weber2020robust}, for different values of the margin $\varepsilon$. The classification accuracy is the average of five independent random trials. The stopping criterion is 
  either a $100\%$ classification accuracy or the theoretical upper bound in~\Cref{thm:hyperbolic_perceptron} (Figure $(a)$), and Theorem 3.1 in \citep{weber2020robust} (Figure $(b)$). }
  \label{fig:weber_not_converge}
\end{figure}

\section{Proof of \Cref{prop:three_conditions}}
Let $\phi(w) =[b, w_{\bE}, \sqrt{\alpha_{\bS}}\phi_{d_{\bS}}(\frac{1}{\sqrt{C_{\bS}}}w_{\bS}), \sqrt{\alpha_{\bH}} \phi_{d_{\bH}}(Rw_{\bH})]= \sum_{n \in [N]} \beta_n M \psi(x_n)$. We now  consider the norm constraint for each component separately.

The parameters for Euclidean component can be written as 
\[
	b = \sum_{n\in [N]} \beta_n, \ w_{\bE} = \sum_{n \in [N]} \beta_n  x_{\bE, n}
\]
The distance-based Euclidean classifier asks for a vector such that $\norm{w_{\bE}}_2 = \alpha_{\bE}$. We can impose this condition as a quadratic equality constraint on the vector $\beta = (\beta_1, \ldots, \beta_N)$ as follows
\[
\norm{w_{\bE}}^2 = \beta^\T K_{\bE} \beta = \alpha^2_{\bE},
\]
where $K_{\bE} = \big( \langle x_{\bE,i}, x_{\bE,j} \rangle \big)_{i,j \in [N]}$. The parameter of the spherical component can be written as
\[
	\phi_{d_{\bS}}(\frac{1}{ \sqrt{C_{\bS}} }w_{\bS}) = \sum_{n \in [N]} \beta_n \phi_{d_{\bS}}(\sqrt{C_{\bS}} x_{\bS, n})
\]
A distance-based spherical classifier requires 
\[
w_{\bS}: \norm{w_{\bS}}_2 = \sqrt{C_{\bS}},  \ \mbox{i.e.,} \ \phi_{d_{\bS}}(\frac{1}{ \sqrt{C_{\bS}} }w_{\bS})^\T \phi_{d_{\bS}}(\frac{1}{ \sqrt{C_{\bS}} }w_{\bS}) = \mathrm{asin}(1).
\] 
This is imposed by the following quadratic constraint
\[
\norm{\phi_{d_{\bS}}(\frac{1}{\sqrt{C_{\bS}}} w_{\bS})}^2 = \beta^\T K_{\bS} \beta = \frac{\pi}{2},
\]
where $K_{\bS} = \big( \mathrm{asin}( C_{\bS} \langle x_{\bS,i}, x_{\bS,j} \rangle ) \big)_{i,j \in [N]}$. Finally, we can write the hyperbolic component as follows
\[
	\phi_{d_{\bH}}(R w_{\bH}) = \sum_{n \in [N]} \beta_n M_{d_{\bH}} \phi_{d_{\bH}} (\frac{1}{R}H x_{\bH, n}).
\]
The distance-based hyperbolic classifier $w_{\bH}$ must satisfy the norm constraint of $[R w_{\bH}, R w_{\bH}] = -R^2C_{\bH}$. Consequently, we must have
\[
	\phi_{d_{\bH}}(R w_{\bH})^\T M_{d_{\bH}} \phi_{d_{\bH}}(R H w_{\bH})  = \mathrm{asinh}(-R^2C_{\bH}).
\]
\begin{lemma}\label{lem:hyperbolic_phi_and_H}
$\phi_{d_{\bH}}(R H w_{\bH}) = \sum_{i \in [N]} \beta_n M_{d_{\bH}} \phi_{d_{\bH}} (\frac{1}{R} x_{\bH, n})$.
\end{lemma}
\begin{proof}
\begin{align*}
\mathrm{asinh}(x_{\bH}^\T w_{\bH})  &= \phi_{d_{\bH}}(\frac{1}{R} x_{\bH})^\T M_{d_{\bH}} \phi_{d_{\bH}}(R w_{\bH}) \\
&\stackrel{\mathrm{(a)}}{=}  \sum_{n \in [N]} \beta_n \mathrm{asin}( \frac{1}{R^2}[ x_{\bH}, x_{\bH, n}]) \\
&= \sum_{n \in [N]} \beta_n \phi_{d_{\bH}}(\frac{1}{R} Hx_{\bH})^\T M_{d_{\bH}} M_{d_{\bH}} \phi_{d_{\bH}}(\frac{1}{R} x_{\bH, n}) \\
&=  \phi_{d_{\bH}}(\frac{1}{R} Hx_{\bH}) ^\T M_{d_{\bH}} \sum_{n \in [N]} \beta_n M_{d_{\bH}} \phi_{d_{\bH}}(\frac{1}{R} x_{\bH, n}) 
\end{align*}
where $\mathrm{(a)}$ is due to $\phi_{d_{\bH}}(R w_{\bH}) = \sum_{n \in [N]} \beta_n M_{d_{\bH}} \phi_{d_{\bH}} (\frac{1}{R}H x_{\bH, n})$.
From 
\[
\mathrm{asinh}(x_{\bH}^\T w_{\bH})= \phi_{d_{\bH}}(\frac{1}{R} H x_{\bH}) ^\T M_{d_{\bH}} \phi_{d_{\bH}}(RH w_{\bH}),
\]
we have $\phi_{d_{\bH}}(RH w_{\bH}) = \sum_{n \in [N]} \beta_n M_{d_{\bH}} \phi_{d_{\bH}}(\frac{1}{R} x_{\bH, n})$.
\end{proof}
From \Cref{lem:hyperbolic_phi_and_H}, we have
\begin{align*}
\mathrm{asinh}( [R w_{\bH},R w_{\bH}] ) &= \phi_{d_{\bH}}(R w_{\bH})^\T M_{d_{\bH}} \phi_{d_{\bH}}(R H w_{\bH}) \\
&= \sum_{i,j} \beta_i \beta_j \mathrm{asinh}(\frac{1}{R^2}[ x_{\bH,i}, x_{\bH, j}] )\\
&= \mathrm{asinh}(-R^2C_{\bH}).
\end{align*}
The kernel matrix $K_{\bH} = \big(  \mathrm{asinh}(\frac{1}{R^2}[ x_{\bH,i}, x_{\bH, j}] ) \big)_{i,j \in [N]}$ is an indefinite matrix. Therefore, we have the following non-convex second-order equality constraint
\[
\phi_{d_{\bH}}^\T (R w_{\bH}) M_{d_{\bH}} \phi_{d_{\bH}} (R H w_{\bH}) = \beta^\T K_{\bH} \beta = \mathrm{asinh}(-R^2C_{\bH}).
\]

\section{Supplementary Numerical Results on Real-World Datasets} \label{sec:Additional_Real_World_Datasets}
\subsection{Datasets}
We used the following publicly available datasets: 
\begin{enumerate}
    \item Lymphoma patient dataset~\citep{Hodgkin2020Lymphoma}\footnote{\href{https://www.10xgenomics.com/resources/datasets/hodgkins-lymphoma-dissociated-tumor-targeted-compare-immunology-panel-3-1-standard}{https://www.10xgenomics.com/resources/datasets/hodgkins-lymphoma-dissociated-tumor-targeted-compare-immunology-panel-3-1-standard}}. Human dissociated lymph node tumor cells of a 19-year-old male Hodgkins Lymphoma patient were obtained by 10x Genomics from Discovery Life Sciences. 
    \item Lymphoma healthy donor dataset~\citep{PBMCs2020HealthyDonor}\footnote{\href{https://www.10xgenomics.com/resources/datasets/pbm-cs-from-a-healthy-donor-targeted-immunology-panel-3-1-standard}{https://www.10xgenomics.com/resources/datasets/pbm-cs-from-a-healthy-donor-targeted-immunology-panel-3-1-standard}}. Human peripheral blood mononuclear cells (PBMCs) of a healthy female donor aged 25 were obtained by 10x Genomics from AllCells. This dataset contains $13,410$ samples (combined) and each for a class (binary classification). The dimension of each cell gene expression vector is $1,020$.
       \item Blood cells landmark dataset~\citep{zheng2017massively}\footnote{\href{https://www.nature.com/articles/ncomms14049}{https://www.nature.com/articles/ncomms14049}}. This dataset contains the gene expression data for (1) B cells, (2) Cd14 monocytes, (3) Cd34 monocytes, (4) Cd4 t helper cells, (5) Cd56 natural killer cells, (6) Cytotoxic T cells, (7) Memory T cells, (8) Naive cytotoxic cells, (9) Native T cells, and (10) Regulatory T cells. It contains $94,655$ samples from a total of $10$ classes. The dimension of each cell gene expression vector is $965$.
    \item MNIST,\footnote{\href{http://yann.lecun.com/exdb/mnist/}{http://yann.lecun.com/exdb/mnist/}} which contains images of handwritten digits~\citep{lecun1998gradient}.
    \item Omniglot,\footnote{ \href{https://github.com/brendenlake/omniglot}{https://github.com/brendenlake/omniglot}} which contains handwritten characters from a variety of world alphabets~\citep{lake2015human}.
    \item CIFAR-100\footnote{ \href{https://www.cs.toronto.edu/~kriz/cifar.html}{https://www.cs.toronto.edu/$\sim$kriz/cifar.html}}~\citep{krizhevsky2009learning}.
\end{enumerate}

\subsection{Omniglot, MNIST, and CIFAR-100 Datasets}
As suggested in~\citep{salakhutdinov2008quantitative} and~\citep{burda2015importance}, we down-sampled images in MNIST and Omniglot datasets to $28 \times 28$ pixels and preprocessed them through a dynamic-binarization procedure.

Following the procedure introduced in the work on mixed-curvature VAEs~\citep{Skopek2020Mixed-curvature} (under license ASL 2.0)\footnote{The code can be found at \href{https://github.com/oskopek/mvae}{https://github.com/oskopek/mvae}.} we embedded both datasets into the product space $\bE^{2}\times \bS^{2}\times \bH^{2}$. For the MNIST dataset, the curvatures of the hyperbolic and spherical spaces were set to $-0.129869$ and $0.286002$. For the Omniglot dataset, these curvatures were set to $-0.173390$ and $0.214189$. Our experiments reveal that the difference in the log-likelihood metric --- used to compare the quality of mixed-curvature and Euclidean embeddings --- is very small (see the results reported in~\citep{Skopek2020Mixed-curvature} and reproduced in Table~\ref{tab:LL}). 

\begin{table}[t!]
  \caption{Estimated marginal log-likelihood after embedding the data into lower-dimensional spaces.}
  \label{tab:LL}
  \centering
  \begin{tabular}{ccc}
    \toprule
     & MNIST & Omniglot\\
    \midrule
    $\bE^6$ & $-96.88\pm 0.16$ & $-136.05\pm 0.29$ \\
    $\bE^{2}\times \bS^{2}\times \bH^{2}$ & $-96.71\pm 0.19$ & $-135.93\pm 0.48$ \\
  \bottomrule
\end{tabular}
\end{table}

To enable $K$-class perceptron classification, we used $K$ binary classifiers --- represented by parameters $w^{(k)}, k\in [K]$ --- that were independently trained on the same training set to separate each single class from the remaining classes. For each classifier, we first transformed the resulting prediction scores into probabilities via Platt's scaling technique~\citep{platt1999probabilistic}. The predicted labels are decided by a maximum a posteriori criteria, using the probability of each class.

The embedded points --- from different classes --- are not guaranteed to be linearly separable. Hence, we restricted the perceptron algorithms to terminate after going through a fixed  number of passes. Subsequently, we computed the Macro F1 scores to determine the quality of the learned linear classifiers. For simplicity, we only report ternary classification results for MNIST and Omniglot datasets: We choose $500$ points from three randomly chosen classes from MNIST, and $20$ points from three randomly chosen classes from Omniglot.

\begin{figure}[t!]
  \center
  \includegraphics[width=1 \linewidth]{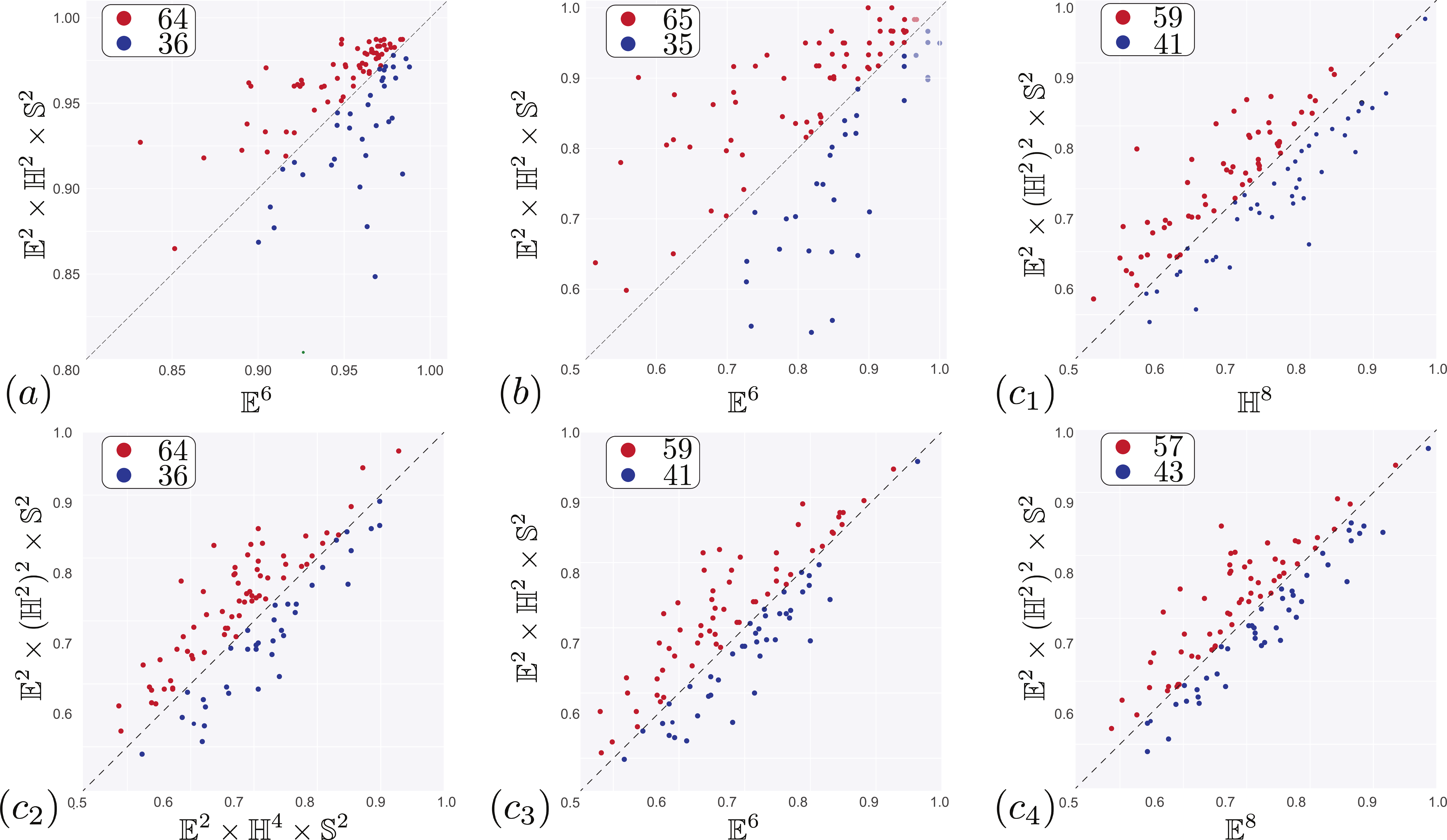}
  \caption{Comparison of Macro F1 scores of the Euclidean and product space form perceptron on $(a)$ MNIST, $(b)$ Omniglot, and $(c_1)-(c_4)$ CIFAR-100 datasets. The labels on the $x$ and $y$ axis indicate the embedding spaces, and the counts in the top-left-corner indicate how often a binary classifier in one space outperforms that in another.}
  \label{fig:additional_results_macro_f1}
\end{figure}

The performance of ternary classifiers is shown in~\Cref{fig:additional_results_macro_f1} $(a)$ and $(b)$. The results are obtained by randomly selecting $100$ sets of three classes; each point in the figure corresponds to one such combination, and its coordinate value equals the averaged Macro F1 score of three independent runs. Red-colored points indicate better performance of the product space form perceptron, while blue-colored points indicate better performance of the Euclidean perceptron. We observe that the product space form perceptron classifies almost twice as many points with higher accuracy compared to its Euclidean counterpart. The performance gain of the product space form perceptron compared to the Euclidean perceptron (in terms of the average gain of Marco F1 scores) is $0.2\%$ for MNIST and $1.27\%$ for Omniglot. These results show that the proposed product space form perceptron algorithm makes better use of the features from a product space form to perform the learning task. Finally, in~\Cref{fig:additional_results_macro_f1} $(c_1)-(c_4)$, we compare the averaged Macro F1 scores for embedded CIRFAR-100 data points in product space forms with their Euclidean counterparts. These supplementary results also indicate that product space form perceptrons offer improved classification results compared to the Euclidean perceptrons. Although the performance improvements are modest, they may be attributed and increased by further adaptation of the VAE algorithm in~\citep{Skopek2020Mixed-curvature}.

\bibliography{bibfile}

\end{document}